\documentclass{article}

\usepackage[final]{neurips_2024}

\usepackage{times}
\usepackage{soul}
\usepackage{url}
\usepackage[utf8]{inputenc} 
\usepackage[T1]{fontenc}    
\usepackage{hyperref}       
\usepackage{url}            
\usepackage{booktabs}       
\usepackage{amsfonts}       
\usepackage{nicefrac}       
\usepackage{microtype}      
\usepackage{xcolor}         
\usepackage[small]{caption}
\usepackage{graphicx}
\usepackage{amsmath}
\usepackage{subfigure}
\usepackage{amsthm}
\usepackage{amssymb,amsfonts}
\usepackage{algpseudocode}
\usepackage[linesnumbered,ruled,vlined]{algorithm2e}
\usepackage{changepage}
\usepackage{subcaption}
\usepackage{multirow}

\newtheorem{theorem}{Theorem}
\newtheorem{lemma}{Lemma}
\newtheorem{assumption}{Assumption}
\newtheorem{definition}{Definition}
\title{Gradient-Guided Conditional Diffusion Models for Private Image Reconstruction: Analyzing Adversarial Impacts of Differential Privacy and Denoising}

\author{%
\AND{Tao Huang$^{1*}$ \quad Jiayang Meng$^{2*}$ \quad Hong Chen$^2$ \quad Guolong Zheng$^1$}
\AND {Xu Yang$^1$ \quad  Xun Yi$^3$ \quad Hua Wang$^4$} \\
$^1$ School of Computer and Big Data, Minjiang University\\
$^2$ School of Information, Renmin University of China\\
$^3$ RMIT University\\
$^4$ Victoria University\\
\texttt{\{huang-tao, gzheng, xu.yang\}@mju.edu.cn}\\
\texttt{\{jiayangmeng, chong\}@ruc.edu.cn}\\
\texttt{xun.yi@rmit.edu.au}\\
\texttt{hua.wang@vu.edu.au}\\
$^*$ The authors contribute equally to this work.
}

\begin{document}

\maketitle

\begin{abstract}
  We investigate the construction of gradient-guided conditional diffusion models for reconstructing private images, focusing on the adversarial interplay between differential privacy noise and the denoising capabilities of diffusion models. While current gradient-based reconstruction methods struggle with high-resolution images due to computational complexity and prior knowledge requirements, we propose two novel methods that require minimal modifications to the diffusion model's generation process and eliminate the need for prior knowledge. Our approach leverages the strong image generation capabilities of diffusion models to reconstruct private images starting from randomly generated noise, even when a small amount of differentially private noise has been added to the gradients. We also conduct a comprehensive theoretical analysis of the impact of differential privacy noise on the quality of reconstructed images, revealing the relationship among noise magnitude, the architecture of attacked models, and the attacker's reconstruction capability. Additionally, extensive experiments validate the effectiveness of our proposed methods and the accuracy of our theoretical findings, suggesting new directions for privacy risk auditing using conditional diffusion models.
\end{abstract}

\section{Introduction}

In recent years, the fields of machine learning and federated learning have garnered significant attention due to their ability to harness large-scale data for various applications. However, large-scale data often carries private information, such as faces, genders, and so on. One critical concern that has emerged is the privacy of the training data. Gradients, a fundamental component in the optimization process of machine learning models, typically shared among different parties in federated learning, inherently contain private information about the training data. Data reconstruction attacks \cite{b5,b10,b11}, in particular, aim to reconstruct the original training data from the gradients, posing a severe threat to data privacy. When it comes to private image reconstruction, attackers are highly motivated to illicitly acquire private images, especially high-resolution images, due to commercial interests or other reasons, as there are many situations in which high-resolution images are needed. For instance, in fields like radiology, high-resolution images can make the difference between missing and identifying critical patient conditions. Detailed images help in diagnosing diseases, planning treatments, and monitoring patient progress. These images are costly to access and always contain private information.

Existing gradient-based image reconstruction attacks are only effective for low-resolution images, with poor performance or very slow reconstruction speeds for high-resolution images. This is because current methods require repeated differentiation of randomly initialized images \cite{b5,b6}, and as the dimensionality of these images increases, the differentiation becomes more difficult, leading to greater reconstruction errors. Reconstruction methods based on Generative Adversarial Networks (GANs) \cite{b10,b11} or fine-tuned diffusion models \cite{b25,xue2024diffusion} introduce additional computational complexity and require certain prior conditions. For example, the distribution of the initial images must match that of the private images. 

Diffusion models \cite{b12,b13,b14,b15}, inspired by the physical process of diffusion, iteratively add noise to data, break up the original data into pure Gaussian noise, and learn to reverse this process by predicting the added Gaussian noise, thereby generating new data from a Gaussian random vector by denoising the added Gaussian noise step by step. Conditional diffusion models \cite{b16,b17,b18,b19,b20,b21,b22,b23,b24} extend this approach by conditioning the generation process on additional information, allowing for more controlled and targeted data generation. Due to the excellent image-generating ability of diffusion models, recent work \cite{b25} has demonstrated that private gradients can serve as the guidance for fine-tuning pre-trained diffusion models to generate private images and the adversary can obtain a gradient-based fine-tuned diffusion model to conduct an effective image reconstruction attack.

To mitigate the privacy risks associated with gradient leakage, one effective method is to add differentially private noise \cite{b28} to the original gradients such that the adversary could not infer private information from these noisy gradients \cite{b32,b5,b25,b26,b27}. Differential privacy \cite{b28} ensures that the added noise masks the private information, thereby reducing the risk of data reconstruction attacks.   

Based on the noising ability of differential privacy and the denoising ability of diffusion models, we are interested in the following issues: Given the capability of conditional diffusion models to generate high-resolution images, is it possible for an attacker to incorporate stolen or inadvertently leaked gradients into the conditional diffusion model as a guiding condition, thereby reconstructing the original image? Current reconstruction methods require certain priors. Could this prerequisite be eliminated in a conditional diffusion model without compromising the reconstruction quality? In other words, can a conditional diffusion model, starting from a randomly generated noise, utilize stolen or leaked gradients to reconstruct the original private image? Could diffusion models' denoising ability contribute to the data reconstruction attack even if the adversary captures the noisy gradients with injected differentially private noise? By analyzing the interplay between the noise introduced by differential privacy and the denoising ability of diffusion models, we can better understand their adversarial impacts.

In this paper, we delve into how to establish a gradient-guided conditional diffusion model to reconstruct private images and the adversarial impacts between differential privacy's noising mechanism and the denoising ability of diffusion models. We aim to provide a comprehensive analysis of how these two factors interact and influence the overall privacy of private data. Through this investigation, we seek to advance the understanding of differential privacy-preserving techniques in the context of protecting private gradients. In detail, we first discuss how the attacker obtains conditional diffusion models based on gradients to reconstruct private images if they capture the original or differentially private gradients. Then, we study the adversarial impacts of differential privacy's noising mechanism and the denoising ability of diffusion models by exploring how the differentially private noise scale influences the Jensen gap between the reconstructed and private images. This Jensen gap is highly correlated to the reconstruction ability of conditional diffusion models.

In summary, our contributions are:
\begin{itemize}
    \item We propose two methods for constructing conditional diffusion models using leaked or stolen gradients to reconstruct private images. These methods require only minimal modifications to a few steps in the diffusion model's generation process to reconstruct high-quality private images. Moreover, these methods do not require the attacker to have prior knowledge; as long as gradient information is captured, the attacker can reconstruct the private image starting from the initial noise image.
    \item We conduct a detailed theoretical analysis of these methods, illustrating the impact of differential privacy noise magnitude and the type of attacked model on the quality of reconstructed private images. The analysis reveals the relationship among the denoising process of diffusion models, the magnitude of differential privacy noise, and the attacker's reconstruction capability. It also explains why different attacked models exhibit varying levels of vulnerability to the same reconstruction attack.
    \item We conduct extensive experiments to demonstrate the excellent capability of conditional diffusion models in data reconstruction. Our experiments also validate the correctness of the theoretical analysis, suggesting new research directions for privacy risk auditing using conditional diffusion models.
\end{itemize}

\section{Related Work and Preliminary}

\subsection{Gradient Leakage and Image Reconstruction}\label{subsection2.1}

In deep learning, model training is executed in parallel across different nodes sometimes, and synchronization occurs through gradient exchange, for example, federated learning \cite{b1}. Recent studies on the risks associated with reconstructing private training images in distributed or federated scenarios reveal that adversaries can steal gradients by eavesdropping on communication channels \cite{wu2024concealing}, impersonating federated learning participants \cite{hitaj2017deep}, and so on. It is indeed possible to extract private training data from the shared or stolen gradients \cite{b2,b3,b4,b5,b6}. Namely, the adversary can recover the training inputs from these shared or stolen gradients, questioning the safety of gradient sharing in preserving data privacy. Usually, the attacker optimizes a dummy input to minimize the difference between its gradient in the attacked model and the leaked gradient \cite{b5,b6}. The optimization can be formulated as:
\begin{equation}
    \begin{split}
    \min _{\mathbf{x}^{\prime}} \mathcal{L}(\nabla F\left( \mathbf{x}^{\prime},W \right),\nabla F\left( \mathbf{x},W \right)) = \min _{\mathbf{x}^{\prime}}\|\nabla F\left( \mathbf{x}^{\prime},W \right) -\nabla F\left( \mathbf{x},W \right) \|^2
    \end{split}
\label{eq1}
\end{equation}
, where $\mathbf{x}^{\prime}$ is the dummy input initialized randomly, $\nabla F\left( \mathbf{x}^{\prime}, W \right)$ is the gradient of the loss with respect to the image $\mathbf{x}^{\prime}$ for the current model weights $W$, and $\nabla F\left( \mathbf{x}, W \right)$ is the leaked gradient that the attacker obtains. $F$ is the loss function of the attacked model and $\mathcal{L}(\nabla F\left( \mathbf{x}^{\prime},W \right),\nabla F\left( \mathbf{x},W \right))$ is the reconstruction loss. The reconstruction loss adopts the mean square error (MSE) loss in Eq.(\ref{eq1}).

In the realm of reconstructing training inputs with shared gradients based on the reconstruction loss, various advanced methodologies \cite{ig,gi,gias,ggl,gifd} have emerged. \cite{ig} achieves efficient image recovery through a meticulously designed loss function. \cite{gi,gias,ggl} propose the utilization of Generative Adversarial Networks (GANs) acting as prior information, which offers an effective approximation of natural image spaces, significantly enhancing attack methodologies. \cite{gifd} exhibits outstanding generalization capabilities in practical settings, underscoring its adaptability and wider applicability across various scenarios. \cite{b25} reconstructs high-resolution images accurately by fine-tuning diffusion model, but adds computational complexity and requires specific prior conditions.

\subsection{Diffusion Models and Conditional Diffusion Models}\label{subsection2.2}

Diffusion models \cite{b13,b29,b30,b31}, inspired by the physical process of diffusion, which describes the movement of particles from regions of higher concentration to lower concentration, are a class of generative models in machine learning that have gained significant attention due to their ability to produce high-quality, diverse samples. Diffusion models work by gradually adding noise to data over a series of steps, transforming the data into a pure noise distribution. This process is then reversed to generate data samples from the noise. The two key processes involved are the forward diffusion process and the reverse generation process. 

Specifically, the stochastic differential equation (SDE) for the data noising process $\mathbf{x}_t, t \in [0,T]$ in the following form \cite{b30}:
\begin{equation}
    \text{d}\mathbf{x} = -\frac{\beta(t)}{2}\mathbf{x} \text{d}t + \sqrt{\beta(t)} \text{d}\mathbf{w}
    \label{forwardSDE}
\end{equation}
, where $\beta(t): \mathbb{R} \mapsto \mathbb{R}>0$ is the continuous noise schedule of the process, typically taken to be a monotonically increasing linear function of $t$ \cite{b13}, and $\mathbf{w}$ is the standard Wiener process. The data distribution is defined as $\mathbf{x}_0 \sim p_{\text{data}}$. A simple, tractable distribution (e.g. isotropic Gaussian) is achieved when $t = T$, i.e. $\mathbf{x}_T \sim \mathcal{N}(0,I)$. To recover the data-generating distribution starting from the tractable distribution, it can be achieved via the corresponding reverse SDE of Eq.(\ref{reverseSDE}):
\begin{equation}
    \text{d}\mathbf{x} =\left( -\frac{\beta(t)}{2}\mathbf{x} - \beta(t) \nabla_{\mathbf{x}_t} \log p_t(\mathbf{x}_t) \right)\text{d}t + \sqrt{\beta(t)} \text{d}\mathbf{\bar{w}}
    \label{reverseSDE}
\end{equation}
, where $\text{d}t$ corresponds to time running backward and $\text{d}\mathbf{\bar{w}}$ is the standard Wiener process running
backward. The time-dependent drift function $\nabla_{\mathbf{x}_t} \log p_t(\mathbf{x}_t)$ is closely related to the conditional diffusion models which will be discussed later.

In the discrete settings with $N$ bins, we define $\mathbf{x}_t = \mathbf{x}(\frac{tT}{N}), \beta_t = \beta(\frac{tT}{N})$ and we can get the discrete version of the forward and corresponding reverse SDE, which hints the Denoising Diffusion Probabilistic Models (DDPM) \cite{b13}. If $\mathbf{x}_0 \sim p_{\text{data}}$ represents the original data, the process generates a sequence $\{ \mathbf{x}_1, \mathbf{x}_2, \ldots, \mathbf{x}_T \}$ where each $\mathbf{x}_t$ is a noisier version of $\mathbf{x}_{t-1}$, and $\mathbf{x}_T$ is indistinguishable from pure noise. This can be mathematically represented as \cite{b13,b29}:
\begin{equation}
\mathbf{x}_t=\sqrt{\alpha_t}\mathbf{x}_{0}+\sqrt{1-\alpha_t}\epsilon _t, \epsilon_t \sim \mathcal{N}\left( 0,I \right)
\label{forward}
\end{equation}
, where $\alpha_t: = \prod_{s=1}^t \left(1 - \beta_s \right)$, $\{\beta_t \}_{t=1}^{T}$ are fixed variance schedules and $\epsilon_t$ is the Gaussian noise from a standard normal distribution.

The reverse process is a denoising process. It involves learning a parameterized model $p_{\theta}\left( \mathbf{x}_{t-1}|\mathbf{x}_t \right)$ to reverse the forward steps. It is modeled by a neural network trained to predict the noise $\epsilon_t$ that is added at each step in the forward process. The sampling methods of different diffusion models differ. In DDPM, the reverse process is a stochastic Markovian process, which is formulated as \cite{b13,b29}:
\begin{equation}
\mathbf{x}_{t-1}=\frac{1}{\sqrt{1- \beta_t}}\left( \mathbf{x}_t-\frac{\beta_t}{\sqrt{1-\alpha_{t}}}\epsilon _{\theta}\left( \mathbf{x}_t,t \right) \right) + \sigma_t \mathbf{z}
\label{eq2}
\end{equation}
, where $\epsilon _{\theta}\left( \mathbf{x}_t,t \right)$ is a prediction of $\epsilon_t$ parameterized by $\theta$, and $\mathbf{z} \sim \mathcal{N}\left( 0,I \right)$. Meanwhile, Denoising Diffusion Implicit Model (DDIM) \cite{b30} is proposed as an alternative non-Markovian denoising process that has a distinct sampling process as follows \cite{b30}:
\begin{equation}
    \begin{split}
    \mathbf{x}_{t-1}=\sqrt{\alpha_{t-1}} f_{\theta}\left(\mathbf{x}_{t}, t\right)+\sqrt{1-\alpha_{t-1}-\sigma_{t}^{2}} \epsilon_{\theta}\left(\mathbf{x}_{t}, t\right)+\sigma_{t}^{2} \mathbf{z}
    \end{split}
\label{DDIM_1}
\end{equation}
, where $\mathbf{z} \sim \mathcal{N}\left( 0,I \right)$, and $f_{\theta}\left(\mathbf{x}_{t}, t\right)$ is the prediction of $\mathbf{x}_{0}$ at $t$ given $\mathbf{x}_{t}$ and $\epsilon_{\theta}\left(\mathbf{x}_{t}, t\right)$:
\begin{equation}
f_{\theta}\left(\mathbf{x}_{t}, t\right):=\frac{\mathbf{x}_{t}-\sqrt{1-\alpha_{t}} \epsilon_{\theta}\left(\mathbf{x}_{t}, t\right)}{\sqrt{\alpha_{t}}}.
\label{DDIM_2}
\end{equation}
Eq.(\ref{DDIM_2}) is derived from Eq.(\ref{forward}). In DDIM, the reverse diffusion process is deterministic when $\sigma_{t}=0$, resulting in its ability to generate images stably.

\begin{figure*}
  \centering
  \includegraphics[scale=0.6]{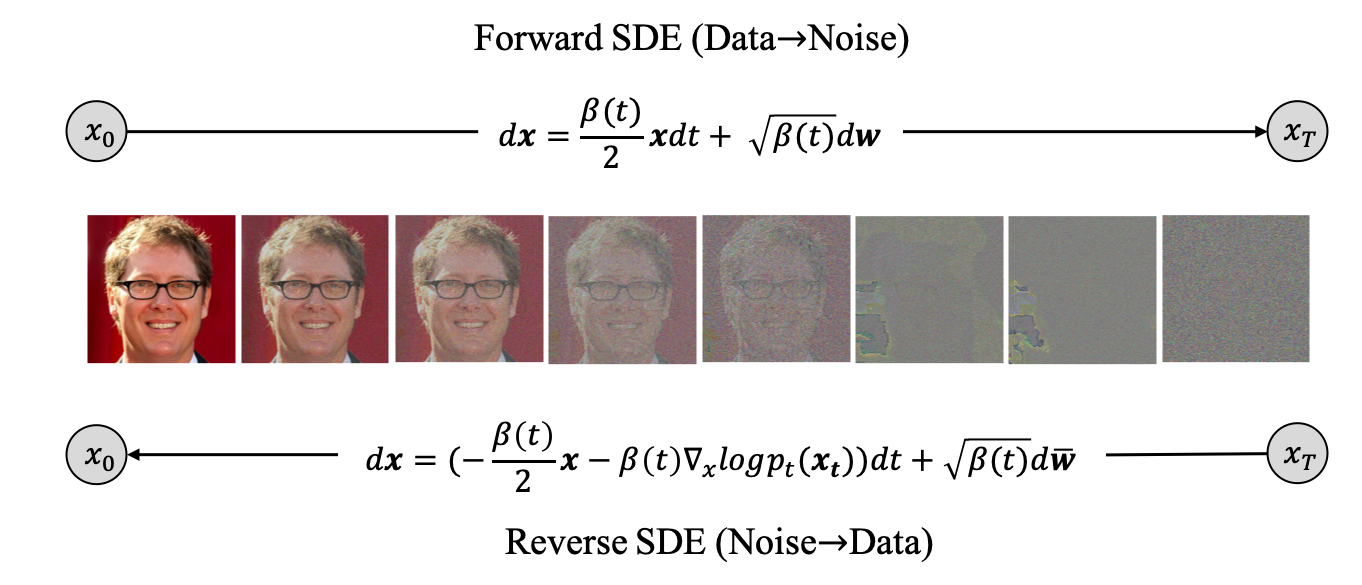}
  \caption{Forward and reverse process of diffusion models}
  \label{fig1_main}
\end{figure*}

Training diffusion models involves optimizing the parameters $\theta$ of the reverse process, which tries to minimize the difference between the noise added in the forward step $t$ and the predicted noise in the reverse step $t$, in simple terms, makes the generated image similar to the input one. The optimization objective can be expressed as \cite{b13,b29}:
\begin{equation}
\min _{\theta}\mathbb{E} _{t,\mathbf{x}_0,\epsilon _t}\left[ ||\epsilon _t-\epsilon _{\theta}\left( \mathbf{x}_t,t \right) ||^2 \right].
\label{eq3}
\end{equation}

Conditional diffusion models \cite{b16,b17,b18,b22,b23,b24} introduce the given condition or measurement $\mathbf{y}$ with an additional likelihood term $p(\mathbf{x}_t|\mathbf{y})$. Considering the Bayes rule,
\begin{equation}
    \nabla_{\mathbf{x}_t} \log p(\mathbf{x}_t|\mathbf{y}) = \nabla_{\mathbf{x}_t} \log p(\mathbf{x}_t) + \nabla_{\mathbf{x}_t} \log p(\mathbf{y} | \mathbf{x}_t).
\end{equation}
Leveraging the diffusion model as the prior whose reverse SDE is expressed as Eq.(\ref{reverseSDE}), it is straightforward to obtain the reverse SDE of the conditional diffusion models, namely the reverse diffusion sampler for sampling from the posterior distribution:
\begin{equation}
    \text{d}\mathbf{x} =\left( -\frac{\beta(t)}{2}\mathbf{x} - \beta(t) \left( \nabla_{\mathbf{x}_t} \log p(\mathbf{x}_t) + \nabla_{\mathbf{x}_t} \log p(\mathbf{y} | \mathbf{x}_t) \right) \right)\text{d}t + \sqrt{\beta(t)} \text{d}\mathbf{\bar{w}}
    \label{reverseSDE2}
\end{equation}
, where $p(\mathbf{x}_t)$ is the prior that is determined by $\epsilon_{\theta}\left(\mathbf{x}_{t}, t\right)$. For various scientific problems, we have a partial measurement $\mathbf{y}$ that is derived from $\mathbf{x}_{t}$, and the distribution is denoted as $p(\mathbf{y} | \mathbf{x}_t)$.

\subsection{Differential Privacy}\label{subsection2.3}
Differential privacy \cite{b28} is a mathematical framework designed to provide strong privacy guarantees for individual data entries in a dataset. It achieves this by ensuring that the output of a computation does not significantly change when any single data entry is modified. This ensures that sensitive information about individuals remains protected, even when an adversary has access to the output of the computation.

\begin{definition}[\textbf{Differential privacy} \cite{b28}]
A randomized mechanism $\mathcal{M}$ satisfies $(\varepsilon, \delta)$-differential privacy (DP) if for any two adjacent datasets $\mathcal{D}$ and $\mathcal{D}^{\prime}$ that differ by a single individual's data, and for all $S \subseteq \text{Range}(\mathcal{M}) $(the set of all possible outputs of $\mathcal{M}$), the following inequality holds:
    $\Pr[\mathcal{M}(\mathcal{D}) \in S] \leq e^\varepsilon \cdot \Pr[\mathcal{M}(\mathcal{D}^{\prime}) \in S] + \delta$.
\end{definition}

In the context of machine learning, gradients are computed during the training process to update the model's parameters. To protect these gradients, differentially private noise (such as Gaussian noise) is added to the gradients before they are used for the parameter updates \cite{b32,b5,b25,b26,b27}. The method is well known as Differentially Private Stochastic Gradient Descent (DP-SGD). The typical four steps undergo: (1) Compute the gradient of the loss function with respect to model parameters; (2) Clip the gradient to limit its sensitivity; (3) Add Gaussian noise to the clipped gradient; (4) Update the model parameters using the noisy gradient. Formally, for a gradient $g$ with sensitivity $\Delta_g$, the DP-SGD update model's parameters with noisy gradients $g^{\prime} = g + \mathcal{N}(0,\sigma^2 I)$, where $\sigma$ is chosen via Lemma \ref{gaussian} according to the Gaussian mechanism formula to ensure $(\epsilon, \delta)$-DP.

\begin{lemma}[\textbf{Gaussian Mechanism for Differential Privacy} \cite{b28}]
Let $ \mathcal{M}: \mathcal{D} \rightarrow \mathbb{R}^k $ be a function with  $\ell_2$-sensitivity $\Delta_2 \mathcal{M} = \| \mathcal{M}(D) - \mathcal{M}(D^{\prime}) \|$ which measures the maximum change in the Euclidean norm of $\mathcal{M}$ for any two adjacent datasets $D$ and $D^{\prime}$ that differ by a single individual's data. The Gaussian Mechanism adds noise drawn from $\mathcal{N}(0, \sigma^2 I)$ to the output of $\mathcal{M}$ and provides $(\varepsilon, \delta)$-differential privacy if $\sigma \geq \frac{\Delta_2 \mathcal{M} \cdot \sqrt{2 \log(1.25 / \delta)}}{\varepsilon}$.
\label{gaussian}
\end{lemma}

\section{Reconstruction of Private Images}

\subsection{Gradients as Condition Measures}
Machine learning models, particularly deep learning models, are often susceptible to various types of attacks that can compromise the confidentiality of the training data. One such attack is the image reconstruction attack, which leverages leaked gradients during the training process to steal private images. When a model is trained using gradient-based optimization methods, such as stochastic gradient descent (SGD), the gradients of the loss function with respect to the model parameters are computed and used to update the parameters. These gradients inherently carry information about the training data. If an attacker gains access to these gradients, they can potentially reverse-engineer the images that were used to compute them.

Mathematically, consider a machine learning model $F(\mathbf{x}; W)$ parameterized by $W$. Given a private data $\mathbf{x} \in \mathbb{X}$, the original gradient $g_0(\mathbf{x})$ is given by:
\begin{equation}
    g_0(\mathbf{x}) = \nabla_W F(\mathbf{x}; W).
\label{sgd}
\end{equation} 
Assuming the attacker obtains $g_0(\mathbf{x})$ and has access to $F(\mathbf{x}; W)$ (i.e., the attacker can calculate the gradient of an arbitrary data by feeding $F(\mathbf{x}; W)$ with the data $\mathbf{x}$), the attacker can iteratively update a dummy image $\mathbf{x}^{\prime}$ which is randomly initialized to minimize the difference between the gradient of $\mathbf{x}^{\prime}$ and the leaked gradient $g_0(\mathbf{x})$, effectively reconstructing the original image $\mathbf{x}$. 

To mitigate the privacy risks posed by gradient leakage, differential privacy (DP) can be employed. DP ensures that the inclusion or exclusion of a single training sample does not significantly affect the model's output, thereby providing privacy guarantees for individual data record. In the context of protecting private gradients, DP can be implemented by adding noise to the gradients. In this paper, we consider the Gaussian mechanism, where differentially private noise is sampled from the Gaussian distribution. Specifically, the noisy gradient, denoted as ${g} = g_0(\mathbf{x}) + \mathcal{N}(0, \sigma^2 I)$, where $\mathcal{N}(0, \sigma^2 I)$ represents Gaussian noise, with $\sigma^2$ being the noise scale determined by the privacy parameters in DP. The Gaussian mechanism in DP has been demonstrated to be effective in mitigating the risk of image reconstruction attacks \cite{b5,b25,b26,b27}.

Existing work \cite{b25} has shown that a pre-trained diffusion model can be fine-tuned via gradient guidance, allowing the fine-tuned diffusion model to recover high-resolution private images. In our view, $\nabla_W F(\mathbf{x}; W)$ can be considered a mapping $\mathbf{x} \mapsto g$, and the leaked gradients ${g}$ can serve as the measurement or condition. In this section, we broaden the discussion to explore how an attacker might guide a pre-trained diffusion model $\epsilon_{\theta}(\mathbf{x})$ to generate images through a series of denoising reverse steps if they capture the gradients $g$ without fine-tuning diffusion model itself, which adds computational complexity. Specifically, we investigate how an attacker could establish a conditional diffusion model $\epsilon_{\theta}(\mathbf{x} | {g})$. Moreover, we examine the adversarial relationship between the data reconstruction capability of conditional diffusion models and the noise-adding ability of differential privacy. In other words, we analyze how the noise scale $\sigma^2$ affects the data reconstruction ability of diffusion models.

When establishing the conditional diffusion model, the attacker designs an attack loss function $\mathcal{L}_{\text{attack}}\left( g({\mathbf{x}}), g; \theta \right)$, where $\mathbf{x} \sim \epsilon_{\theta}(\mathbf{x} | {g})$, to measure the difference between the gradients $g({\mathbf{x}})$ and $g$. Typically, the attack loss is symmetric, meaning that $\mathcal{L}_{\text{attack}}\left( g({\mathbf{x}}), g; \theta \right) = \mathcal{L}_{\text{attack}}\left( g, g({\mathbf{x}}); \theta \right)$. The attacker's objective is to establish the conditional diffusion model $\epsilon_{\theta}(\mathbf{x} | {g})$ and generate $\mathbf{x} \sim \epsilon_{\theta}(\mathbf{x} | {g})$ that minimizes $\mathcal{L}_{\text{attack}}\left( g({\mathbf{x}}), g; \theta \right)$. For simplicity, $\mathcal{L}_{\text{attack}}\left( g({\mathbf{x}}), g; \theta \right)$, $\mathcal{L}\left( g({\mathbf{x}}) \right)$, and $\mathcal{L}\left( {\mathbf{x}} \right)$ are used interchangeably to represent the attack loss, which is a composite function. The properties of $\mathcal{L}\left( g({\mathbf{x}}) \right)$ impact the quality of image reconstruction. The adversarial relationship between the denoising ability of $\epsilon_{\theta}(\mathbf{x} | {g})$ and the noise-adding ability of differential privacy is linked to the properties of $\mathcal{L}\left( g({\mathbf{x}}) \right)$, which we discuss further in this paper.

\subsection{Adversarial Impacts of Differential Privacy and Conditional Denoising Diffusion Models} \label{sec:4.2}

In this section, we discuss how to establish the conditional diffusion attack model $\epsilon_{\theta}(\mathbf{x} | {g})$ and the impacts of the scale of differentially private noise on the reconstruction quality if $g$ is injected with differentially private Gaussian noise. When establishing the conditional diffusion attack model via Eq.(\ref{reverseSDE2}), it is difficult to derive $p(\mathbf{y} = g | \mathbf{x}_t)$ since $\mathbf{x}_t$ is time-dependent. However, a tractable approximation for $p(\mathbf{y} = g | \mathbf{x}_t)$ is not difficult to design. Given the posterior mean $\hat{\mathbf{x}}_{0} = f_{\theta}\left(\mathbf{x}_{t}, t\right)$ that can be effectively computed at the intermedia steps in Eq.(\ref{DDIM_2}), we use the following approximation: $P_{\theta}({g} | \mathbf{x}_{t}) \simeq P_{\theta}({g} | \hat{\mathbf{x}}_{0}) = P_{\theta}({g} | \hat{\mathbf{x}}_{0}(\mathbf{x}_{t}))$, which is related to Jensen’s inequality where Jensen gap quantifies the approximation error. 
\begin{definition}[Jensen Gap \cite{b19}]
    Let $\mathbf{x}$ be a random variable with distribution $p(\mathbf{x})$. For some function $f$ that may or may not be convex, the Jensen gap is defined as 
    \begin{equation}
        \mathcal{J}(f, \mathbf{x} \sim p(\mathbf{x})) = \mathbb{E}[f(\mathbf{x})] - f(\mathbb{E}[\mathbf{x}])
    \end{equation}
, where the expectation is taken over $p(\mathbf{x})$.
\end{definition}
The Jensen gap measures the approximation error introduced by the approximation $P_{\theta}({g} | \mathbf{x}_{t}) \simeq P_{\theta}({g} | \hat{\mathbf{x}}_{0}) = P_{\theta}({g} | \hat{\mathbf{x}}_{0}(\mathbf{x}_{t}))$. A smaller Jensen gap indicates a lower approximation error, which in turn suggests improved reconstruction quality of the conditional diffusion attack model $\epsilon_{\theta}(\mathbf{x} | {g})$.

\textbf{DPS-Based Conditional Diffusion Attack Model (DPS-based Method).} 
Under differential privacy, the noisy gradient $g$ follows a Gaussian distribution. Specifically, the gradient $g(\mathbf{x})$ with respect to the reconstructed image $\mathbf{x}$ should also adhere to a Gaussian distribution. For ease of analysis, we assume $g(\mathbf{x}) \sim \mathcal{N}(\Bar{g}_0(\mathbf{x}), \sigma^2 I)$. The attacker's objective is to identify an image $\mathbf{x}$ such that $\Bar{g}_0(\mathbf{x})$ approximates $g_0$. Given the assumption that $g(\mathbf{x}) \sim \mathcal{N}(\Bar{g}_0(\mathbf{x}), \sigma^2 I)$, we can derive the conditional diffusion model $\epsilon_{\theta}(\mathbf{x} | {g})$ based on Diffusion Posterior Sampling (DPS) \cite{b19}. Algorithm \ref{alg: DPS} outlines the reconstruction process, and Theorem \ref{UpperBoundJensen} provides the upper bound of the Jensen Gap. The key guidance is step 8 in Algorithm \ref{alg: DPS}.
\begin{equation}
    \mathbf{x}_{t-1} \leftarrow \mathbf{x}_{t-1}^{\prime} - \zeta_t \nabla_{\mathbf{x}_t}  \| g- g_0(\hat{\mathbf{x}}_0)\|
\label{eq-DPS}
\end{equation}
\begin{algorithm}[htb]
    \caption{DPS-Based Reconstruction Method}
    \label{alg: DPS}
    \KwIn{$N, {g}, \{\zeta_t\}_{t=1}^N, \{\tilde{\sigma}_t\}_{t=1}^N, F(\mathbf{x}; W)$}
     $\mathbf{x}_N \sim \mathcal{N}(0, I)$\\
    \For{$t=N; t--; t \geq 1$}{
         $\hat{s} \leftarrow s_{\theta}(\mathbf{x}_t, t)$\\
         $\hat{\mathbf{x}}_0 \leftarrow \frac{1}{\sqrt{\bar{\alpha}_t}} \left( \mathbf{x}_t - (1 - \bar{\alpha}_t) \hat{s} \right)$\\
         $\mathbf{z} \sim \mathcal{N}(0, I)$\\ 
         $\mathbf{x}_{t-1}^{\prime} \leftarrow \frac{\sqrt{\bar{\alpha}_t} (1-\bar{\alpha}_{t-1})}{1-\bar{\alpha}_{t-1}} \mathbf{x}_t + \frac{\sqrt{\bar{\alpha}_{t-1}} \beta_t}{1 - \bar{\alpha}_t} \hat{\mathbf{x}}_0 + \tilde{\sigma}_t z$\\
         $g_0(\hat{\mathbf{x}}_0) = \nabla_W F(\hat{\mathbf{x}}_0; W)$\\
         $\mathbf{x}_{t-1} \leftarrow \mathbf{x}_{t-1}^{\prime} - \zeta_t \nabla_{\mathbf{x}_t}  \| g- g_0(\hat{\mathbf{x}}_0)\|$}
    \Return $\hat{\mathbf{x}}_0$
\end{algorithm}

\begin{theorem}[Upper Bound of Jensen Gap of Reconstruction Error]
Considering a machine learning model $F(\mathbf{x}; W)$ that is equipped with differentially private Gaussian noise $\mathcal{N}(0, \sigma^2 I)$. Assuming the attacker captures the noisy gradient ${g} = g_0(\mathbf{x}_0) + \mathcal{N}(0, \sigma^2 I)$, where $g_0(\mathbf{x}_0) = \nabla_{W} F(\mathbf{x}_0;W)$ and $\mathbf{x}_0$ is some private image. The attacker establishes a conditional diffusion attack model $\epsilon_{\theta}(\mathbf{x} | {g})$ by Algorithm \ref{alg: DPS} under the assumption that $P_{\theta}({g} | \mathbf{x}_{t}) \simeq P_{\theta}({g} | \hat{\mathbf{x}}_{0}) = P_{\theta}({g} | \hat{\mathbf{x}}_{0}(\mathbf{x}_{t}))$ where $\hat{\mathbf{x}}_{0}(\mathbf{x}_{t})$ is the reconstructed version of ${\mathbf{x}}_{0}$. Under these conditions, the reconstruction error can be quantified by Jensen Gap, which is upper bounded by: 
\begin{equation}
\begin{split}
    &\mathcal{J}\left( \nabla_{W} F(\mathbf{x};W), P_{\theta}(\mathbf{x}_0 | \mathbf{x}_{t}) \right) \\
    &\leq \frac{d}{\sqrt{2\pi \sigma^{2}}} \| \nabla_{\mathbf{x}} \nabla_{W} F(\mathbf{x};W) \| \int \| \mathbf{x}_0 - \hat{\mathbf{x}}_0 \| dP_{\theta}(\mathbf{x}_0 | \mathbf{x}_{t}).
\end{split}
\end{equation}
\label{UpperBoundJensen}
\end{theorem}

\begin{proof}
Analogous to the proof of Theorem 1 in \cite{b19}, the upper bound of Jensen Gap is:
\begin{equation}
\begin{split}
    \mathcal{J}\left( \nabla_{W} F(\mathbf{x};W), P_{\theta}(\mathbf{x}_0 | \mathbf{x}_{t}) \right) & \leq \frac{d}{\sqrt{2\pi \sigma^{2}}} \exp(- \frac{1}{2\sigma^2}) \| \nabla_{\mathbf{x}} \nabla_{W} F(\mathbf{x};W) \| \int \| \mathbf{x}_0 - \hat{\mathbf{x}}_0 \| dP_{\theta}(\mathbf{x}_0 | \mathbf{x}_{t})\\
    & \leq \frac{d}{\sqrt{2\pi \sigma^{2}}} \| \nabla_{\mathbf{x}} \nabla_{W} F(\mathbf{x};W) \| \int \| \mathbf{x}_0 - \hat{\mathbf{x}}_0 \| dP_{\theta}(\mathbf{x}_0 | \mathbf{x}_{t})
\end{split}
\end{equation}

The last inequality is because $\frac{\exp(-\frac{1}{\sigma^2})}{\sigma^2} \leq \frac{1}{\sigma^2}$.
\end{proof}

The upper bound indicates the worst reconstruction quality. The term $\int \| \mathbf{x}_0 - \hat{\mathbf{x}}_0 \| dP_{\theta}(\mathbf{x}_0 | \mathbf{x}_{t})$ is a number determined by the pre-trained diffusion model, the adopted denoising style (e.g. DDPM or DDIM) and the attacked model. This term struggles with the noise scale $\sigma^2$. The interesting thing is that we find a key factor in the upper bound of Jensen gap in Theorem \ref{UpperBoundJensen}, $\| \nabla_{\mathbf{x}} \nabla_{W} F(\mathbf{x}; W) \| = \max_{\mathbf{x}} \| \nabla_{\mathbf{x}} \nabla_{W} F(\mathbf{x}; W) \|$ that is relevant to the attacked model. In other words, the vulnerability of different machine learning models is not the same under the same noise scale and the same adopted pre-trained diffusion model. $\| \nabla_{\mathbf{x}} \nabla_{W} F(\mathbf{x};W) \|$ can be utilized as a metric to evaluate the vulnerability. Here we choose the Frobenius norm and $\| \nabla_{\mathbf{x}} \nabla_{W} F(\mathbf{x};W) \|$ can be calculated as a metric to display the vulnerability of $F(\mathbf{x}; W)$. We give a formal definition.

\begin{definition}[Reconstruction Vulnerability of Machine Learning Models]
    Given a machine learning model $F(\mathbf{x};W)$ parameterized by $W$ where $\mathbf{x}$ is the input, \textit{reconstruction vulnerability} $RV$ of the model $F(\mathbf{x};W)$ with respect to the private dataset $\mathbb{X}$ is 
    \begin{equation}
        RV = \max_{\mathbf{x} \in \mathbb{X}} \| \nabla_{\mathbf{x}} \nabla_{W} F(\mathbf{x}; W) \|.
    \end{equation}
\end{definition}

We want to discover whether the vulnerability of $F(\mathbf{x}; W)$ has a relationship with the value of $RV$. When computing $RV$, it is computationally intensive if the input dimension is high. We consider projecting $\nabla_{W} F(\mathbf{x};W)$ onto some random direction $\mathbf{v}$ and then estimating $\| \nabla_{\mathbf{x}} \left( \mathbf{v}^{T} \nabla_{W} F(\mathbf{x};W) \right)\|$ along the direction of $\mathbf{v}$. Noting that $\mathbf{v}^{T} \nabla_{W} F(\mathbf{x};W)$ is a number, $\| \nabla_{\mathbf{x}} \left( \mathbf{v}^{T} \nabla_{W} F(\mathbf{x};W) \right)\|$ can be calculated efficiently by automatic differentiation. To obtain more accurate estimation, we can estimate $\| \nabla_{\mathbf{x}} \nabla_{W} F(\mathbf{x};W) \|$ given $\mathbf{x}$ by projecting $\nabla_{W} F(\mathbf{x};W)$ onto some random orthogonal directions $\{ \mathbf{v}_j \}_{j=1}^M$. The larger $M$ will show more details. For $F(\mathbf{x}; W)$ trained with $\{ \mathbf{x}_i \}_{i=1}^N$, we can estimate $RV$ by:
\begin{equation}
\begin{split}
    RV &= \mathbb{E}_{\mathbf{v} \sim \mathbf{p}_{\mathbf{v}}} \mathbb{E}_{\mathbf{x} \sim \mathbf{p}_{\text{data}}(\mathbf{x})} \| \nabla_{\mathbf{x}} \left( \mathbf{v}^{T} \nabla_{W} F(\mathbf{x};W) \right)\| \\
    & \simeq \frac{1}{N} \frac{1}{M} \sum_{i=1}^{N} \sum_{j=1}^{M} \| \nabla_{\mathbf{x}} \left( \mathbf{v}_{ij}^{T} \nabla_{W} F(\mathbf{x}_i;W) \right)\|
\end{split}
\end{equation}
, where $\mathbf{p}_{\text{data}}(\mathbf{x})$ is the distribution of the private dataset $\mathbb{X}$. In the experimental section, we will thoroughly discuss the relationship between the vulnerability of different attacked models and the value of $RV$. This analysis provides insight into why different models exhibit varying levels of resistance when subjected to the same reconstruction attack. And the values of $RV$ in the experiments are obtained when $M = 1000$ and $N = 310$.

Subsequently, we discuss the lower bound of the Jensen gap to see the impact of the noise scale on the best reconstruction quality. For step 8 in Algorithm \ref{alg: DPS}, the component $\| g- g_0(\hat{\mathbf{x}}_0)\|$ can be further generalized as an instance of the reconstruction attack loss $\mathcal{L}(g(\mathbf{x}))$. The key step 8 in Algorithm \ref{alg: DPS} can be rewritten as:
\begin{equation}
    \mathbf{x}_{t-1} \leftarrow \mathbf{x}_{t-1}^{\prime} - \zeta_t \nabla_{\mathbf{x}_t}  \mathcal{L}(g(\mathbf{x}_t)).  
\label{eq-DPS2}
\end{equation}

The properties of $\mathcal{L}(g(\mathbf{x}))$ have impacts on the lower bound of Jensen Gap. We derive the lower bound of Jensen Gap under the differentially private Gaussian noise and reveal how Jensen Gap is related to the noise scale. The following Assumptions \ref{assu1}-\ref{assu4} and Lemmas \ref{lemma2}-\ref{lemma3} are needed to derive the lower bound of Jensen Gap under differential privacy.

\begin{assumption}
    $\mathcal{L}$ is $\beta$ convex loss with respect to $g(\mathbf{x})$;
    \label{assu1}
\end{assumption}

\begin{assumption}
    $\mathcal{L}$ is $\alpha$ smooth with respect to $g(\mathbf{x})$;\label{assu2}
\end{assumption}

\begin{assumption}
    $g(\mathbf{x})$ is $L_g$ smooth with respect to $\mathbf{x}$;
    \label{assu3}
\end{assumption}

\begin{assumption}
    The Jacobian matrix $\mathbf{J}_g(\mathbf{x}) = \nabla_{\mathbf{x}} g(\mathbf{x})$ is full rank.
    \label{assu4}
\end{assumption}

\begin{lemma}
    Under the assumption of \textbf{Theorem \ref{UpperBoundJensen}}, if the loss $\mathcal{L}\left( g(\mathbf{x}) \right)$ is $\beta^{\prime}$-strongly convex with respect to $\mathbf{x}$ and $P_{\theta}(\mathbf{x}_0 | \mathbf{x}_{t})$ is a Gaussian distribution with covariance matrix whose eigenvalues are $\{ \sigma_i^2 \}_{i=1}^n$. Then the lower bound of Jensen Gap is bounded by \cite{b19}:
\begin{equation}
    \mathcal{J}\left( \mathcal{L}\left( g(\mathbf{x}) \right), P_{\theta}(\mathbf{x}_0 | \mathbf{x}_{t}) \right) \geq \frac{1}{2} \beta^{\prime} \sum_{i=1}^{n} \sigma_{i}^{2}.
\end{equation}
\label{lemma2}
\end{lemma}

\begin{lemma}
    Let $\phi(\cdot)$ be an isotropic multivariate Gaussian density function with variance matrix $\sigma^2 I$. There exists a constant $L$ such that $\forall \mathbf{x}, \mathbf{y} \in \mathbb{R}^d$,
    \begin{equation}
        \| \phi(\mathbf{x}) - \phi(\mathbf{y}) \| \leq L \| \mathbf{x} - \mathbf{y} \|,
    \end{equation}
    where $L = \frac{d}{\sqrt{2 \pi \sigma^2}} \exp(-\frac{1}{2\sigma^2})$.
    \label{lemma3}
\end{lemma}

\begin{theorem}[Lower bound of Jensen Gap of Reconstruction Error]
Under the conditions of \textbf{Theorem \ref{UpperBoundJensen}}, the Assumptions \ref{assu1}-\ref{assu4} and the Lemmas \ref{lemma2}-\ref{lemma3}, the lower bound of Jensen Gap is bounded by
\begin{equation}
    \mathcal{J}\left( \mathcal{L}\left( g(\mathbf{x}) \right), P_{\theta}(\mathbf{x}_0 | \mathbf{x}_{t}) \right) \geq \frac{1}{2} \left( \beta \lambda_{min}\left( \mathbf{J}_g(\mathbf{x})^T\mathbf{J}_g(\mathbf{x}) \right) - \alpha L_g \right) \sum_{i=1}^{n} \sigma_{i}^{2}
    \label{theo2eq}
\end{equation}
, where $\mathbf{J}_g(\mathbf{x}) = \nabla_{\mathbf{x}} g(\mathbf{x})$ is a full rank Jacobian matrix, $\lambda_{min}(\cdot)$ represents the smallest eigenvalue of $\cdot$, $L_g = \frac{d}{\sqrt{2 \pi \sigma^2}} \exp(-\frac{1}{2\sigma^2})$ and $\{ \sigma_i^2 \}_{i=1}^n$ are the eigenvalues of the covariance matrix of $P_{\theta}(\mathbf{x}_0 | \mathbf{x}_{t})$.
\label{lowerboundDSG}
\end{theorem}

\begin{proof}
    Taking Taylor's expansion of $\mathcal{L}(g(\mathbf{x}))$,
    \begin{equation}
        \mathcal{L}(g(\mathbf{x})) = \mathcal{L}(g(\mathbf{y})) + \left( \mathbf{J}_g(\mathbf{x})^T \nabla_g \mathcal{L}(g(\mathbf{x})) \right)^T (\mathbf{x} - \mathbf{y}) + \frac{1}{2} (\mathbf{x} - \mathbf{y})^T \nabla^2 \mathcal{L}(g(\mathbf{c})) (\mathbf{x} - \mathbf{y})
    \end{equation}
    , where $\nabla^2 \mathcal{L}(g(\mathbf{x}))$ is the Hessian matrix of $\mathcal{L}$ with respect to $\mathbf{x}$ and $\mathbf{c}$ satisfies $g(\mathbf{c}) \in [g(\mathbf{x}), g(\mathbf{y})]$.

    Pay attention to $\nabla^2 \mathcal{L}(g(\mathbf{c}))$,
    \begin{equation}
        \begin{split}
            \nabla^2 \mathcal{L}(g(\mathbf{x})) &= \nabla_{\mathbf{x}}\left( \mathbf{J}_g(\mathbf{x})^T \nabla_g \mathcal{L}(g(\mathbf{x})) \right) \\ 
            &= \mathbf{J}_g(\mathbf{x})^T \nabla_{g}^{2} \mathcal{L}(g(\mathbf{x})) \mathbf{J}_g(\mathbf{x}) + \sum_{i=1}^{m} \nabla \mathcal{L}_i(g(\mathbf{x})) \nabla_{\mathbf{x}}^2 g_i(\mathbf{x}),
        \end{split}
    \end{equation}
    , where $\nabla \mathcal{L}_i(g(\mathbf{x})) = \frac{\partial \mathcal{L}}{\partial g_i(\mathbf{x})}$ and $m$ is the dimension of $g(\mathbf{x})$. $\nabla^2 g_i(\mathbf{x})$ is the Hessian matrix of the $i$-th component of $g(\mathbf{x})$ with respect to $\mathbf{x}$. Since $\mathcal{L}$ is $\beta$ convex attack loss with respect to $g(\mathbf{x})$, we have $\beta I \preceq \nabla_g^2 \mathcal{L}(g(\mathbf{x}))$. Therefore,
    \begin{equation}
        \beta \lambda_{min}\left( \mathbf{J}_g(\mathbf{x})^T\mathbf{J}_g(\mathbf{x}) \right) I \preceq \beta  \mathbf{J}_g(\mathbf{x})^T\mathbf{J}_g(\mathbf{x}) \preceq \mathbf{J}_g(\mathbf{x})^T \nabla_{g}^{2} \mathcal{L}(g(\mathbf{x})) \mathbf{J}_g(\mathbf{x})
    \end{equation}
    , where $\lambda_{min}(\cdot)$ represents the smallest eigenvalue of $\cdot$.
    Since $g(\mathbf{x})$ is $L_g$ smooth, we have $\| \nabla^2 g_i(\mathbf{x}) \| \leq L_g$. And since $\mathcal{L}$ is $\alpha$ smooth with respect to $g(\mathbf{x})$, we have $\| \nabla \mathcal{L}_i(g(\mathbf{x})) \| \leq \alpha$. Therefore,
    \begin{equation}
        \| \sum_{i=1}^{m} \nabla \mathcal{L}_i(g(\mathbf{x})) \nabla_{\mathbf{x}}^2 g_i(\mathbf{x}) \| \leq m \alpha L_g.
    \end{equation}
    Thus,
    \begin{equation}
        \beta \lambda_{min}\left( \mathbf{J}_g(\mathbf{x})^T\mathbf{J}_g(\mathbf{x}) \right) I - m \alpha L_g I \preceq \nabla^2 \mathcal{L}(g(\mathbf{x}))
    \end{equation}

    \begin{equation}
        \sup (\lambda_{min}(\nabla^2 \mathcal{L}(g(\mathbf{c})))) = \beta \lambda_{min}\left( \mathbf{J}_g(\mathbf{x})^T\mathbf{J}_g(\mathbf{x}) \right) - \alpha L_g
    \end{equation}
    Consequently, $\mathcal{L}(g(\mathbf{x}))$ is $\beta^{\prime}$ convex where $\beta^{\prime}= \beta \lambda_{min}\left( \mathbf{J}_g(\mathbf{x})^T\mathbf{J}_g(\mathbf{x}) \right) - \alpha L_g$. When $g(\mathbf{x})$ is Gaussian noisy gradient, $L_g = \frac{d}{\sqrt{2 \pi \sigma^2}} \exp(-\frac{1}{2\sigma^2})$.
\end{proof}

In Theorem \ref{lowerboundDSG}, the term $\frac{1}{2}\left( \beta \lambda_{\text{min}}\left( \mathbf{J}_g(\mathbf{x})^T\mathbf{J}_g(\mathbf{x}) \right) \right) \cdot \sum_{i=1}^{n} \sigma_{i}^{2}$ represents the maximum value of the lower bound. The expression $\lambda_{\text{min}}\left( \mathbf{J}_g(\mathbf{x})^T\mathbf{J}_g(\mathbf{x}) \right)$ can be interpreted as the information loss associated with the mapping $\mathbf{x} \mapsto g$. This suggests that it is impossible to reconstruct an exact replica of the original private image from the leaked gradients. The degree of information loss is scaled by the randomness inherent in the denoising steps, with the scaling factor $\sum_{i=1}^{n} \sigma_{i}^{2}$ determined by the denoising schedule in diffusion models.

In Theorem \ref{lowerboundDSG}, a larger noise scale $\sigma^2$ leads to a smaller value of $L_g$, thereby increasing the lower bound of the Jensen Gap. In other words, the reconstruction error is positively correlated with the noise scale. Specifically, the inequality $L_g \cdot \sum_{i=1}^{n} \sigma_{i}^{2} \leq \frac{d \cdot \sum_{i=1}^{n} \sigma_{i}^{2}}{\sqrt{2 \pi \sigma^2}}$ holds. For a large differentially private noise scale $\sigma^2$, the term $\sum_{i=1}^{n} \sigma_{i}^{2}$ can be chosen to be sufficiently large, ensuring that $\frac{d \cdot \sum_{i=1}^{n} \sigma_{i}^{2}}{\sqrt{2 \pi \sigma^2}}$ remains large, and then the lower bound remains small. A larger value of $\sum_{i=1}^{n} \sigma_{i}^{2}$ typically indicates a greater number of denoising steps in diffusion models. Theoretically, this implies that an attacker can select an appropriate pre-trained diffusion model and increase the number of denoising steps to improve the reconstruction quality.

Moreover, for Eq.(\ref{theo2eq}), the larger $\alpha$ results in the smaller value of the lower bound. Combined with Assumption \ref{assu2} and Eq.(\ref{eq-DPS}), Eq.(\ref{eq-DPS}) is closely related to the smoothness of $\mathcal{L}$ with respect to $g(\mathbf{x})$. The $\alpha$ works with step size $\zeta_i$ in Eq.(\ref{eq-DPS}). Namely, larger $\zeta_i$ is preferred theoretically. We demonstrate this finding in our experiments which are discussed later.

\textbf{DSG-Based Conditional Diffusion Attack Model (DSG-based Method).} 
DPS is based on both the assumption $P_{\theta}({g} | \mathbf{x}_{t}) \simeq P_{\theta}({g} | \hat{\mathbf{x}}_{0}) = P_{\theta}({g} | \hat{\mathbf{x}}_{0}(\mathbf{x}_{t}))$ and the linear manifold assumption. These two assumptions will introduce additional errors as illustrated in \textbf{Theorem \ref{lowerboundDSG}}, leading to a substantial decline in the authenticity of generated samples. To mitigate the additional error of the DPS-based reconstruction method, we propose the DSG-based method (see Algorithm \ref{alg: DSG}) based on \cite{b24} to improve the reconstruction quality of $\epsilon_{\theta}(\mathbf{x} | {g})$. 

\begin{algorithm}[htb]
    \caption{\textbf{DSG-Based Reconstruction Method}}
    \label{alg: DSG}
    \KwIn{$N, \tilde{g}, \text{guidance interval} \; i, \text{step size} \; r, \text{guidance rate} \; m_r,$\\
    $ \{\tilde{\sigma}_t\}_{t=1}^N, \text{denoising network} \; \epsilon_{\theta}(\mathbf{x}_t,t)$}
     $\mathbf{x}_N \sim \mathcal{N}(0, I)$\\
    \For{$t=N$ to $1$}{
         $\hat{s} \leftarrow s_{\theta}(\mathbf{x}_t, t)$\\
         $\hat{\mathbf{x}}_0 \leftarrow \frac{1}{\sqrt{\bar{\alpha}_t}} \left( \mathbf{x}_t - (1 - \bar{\alpha}_t) \hat{s} \right)$\\
         $u_{\theta}(\mathbf{x}_t) = \sqrt{\bar{\alpha}_{t-1}} \hat{\mathbf{x}}_0 + \sqrt{1-\bar{\alpha}_{t-1}-\tilde{\sigma}_t^2} \cdot \epsilon_{\theta}(\mathbf{x}_t,t)$\\
         $d^{*} = -\sqrt{n} \tilde{\sigma}_t \cdot \frac{\nabla_{\mathbf{x}_t} \mathcal{L}(g(\hat{\mathbf{x}}_0; W), \tilde{g})}{\| \nabla_{\mathbf{x}_t} \mathcal{L}(g(\hat{\mathbf{x}}_0; W), \tilde{g})\|}$\\
         $d^{\text{sample}} = \tilde{\sigma}_t \epsilon_t$\\
         $d_m = d^{\text{sample}} + m_r (d^{*} - d^{\text{sample}})$\\
         $\mathbf{x}_{t-1} \leftarrow u_{\theta}(\mathbf{x}_t) + r \frac{d_m}{\| d_m \|} $}
    \Return $\hat{\mathbf{x}}_0$
\end{algorithm}

When it comes to our reconstruction case, the key issue is whether the attack loss $\mathcal{L}\left( g(\mathbf{x}) \right)$ decreases as the reconstructed samples $\{ \mathbf{x}_t \}_{t=T}^{0} \sim \epsilon_{\theta}(\mathbf{x} | {g})$ are gradually generated by Algorithm \ref{alg: DSG}. Namely, whether $\{ \mathcal{L}(\mathbf{x}_{t}) \}_{t=T}^{0}$ is a decreasing sequence. If so, what is the decreasing speed? Is the decreasing speed correlated to the differentially private noise scale? Thus, it is necessary to analyze the upper and lower bound of $\mathcal{L}(\mathbf{x}_{t-1}) - \mathcal{L}(\mathbf{x}_{t})$. To do so, we need the following Lemma \ref{lemma4}.

\begin{lemma}
    Assuming that $\mathcal{L}(g(\mathbf{x}))$ being $\beta^{\prime}$-strongly convex and $\alpha^{\prime}$-smooth with respect to $\mathbf{x}$. Let $H(\mathcal{L}(\mathbf{x}))$ represents the Hessian of $\mathcal{L}$ with respect to $\mathbf{x}$, we have 
    \begin{equation}
        \beta^{\prime} I \preceq H(\mathcal{L}(\mathbf{x})) \preceq \alpha^{\prime} I
    \end{equation}
    \label{lemma4}
    , where $I$ is the identity matrix.
\end{lemma}

\begin{theorem}
    Under the conditions of Lemma \ref{lemma4}, if the conditional denoising diffusion model decreases the attack loss whose denoising steps go as Algorithm \ref{alg: DSG}, then the denoising steps in Algorithm \ref{alg: DSG} results in a decreasing sequence $\{ \mathcal{L}(\mathbf{x}_{t}) \}_{t=T}^{0}$.
    \label{upperDSG}
\end{theorem}

\begin{proof}
    Take Taylor expansion on $\mathcal{L}(\mathbf{x})$ at the point $\mathbf{x}_{t}$ which states:
    \begin{equation}
        \mathcal{L}(\mathbf{x}) = \mathcal{L}(\mathbf{x}_{t}) + \left ( \nabla_{\mathbf{x}} \mathcal{L}(\mathbf{x}_t) \right)^T (\mathbf{x} - \mathbf{x}_{t}) + \frac{1}{2} (\mathbf{x} - \mathbf{x}_{t})^T H(\mathcal{L}(c)) (\mathbf{x} - \mathbf{x}_{t})
    \end{equation}
    , where $c$ is some point between $\mathbf{x}$ and $\mathbf{x}_{t}$.
    Subsequently, we have:
    \begin{equation}
    \begin{split}
        \mathcal{L}(\mathbf{x}_{t-1}) -  \mathcal{L}(\mathbf{x}_{t}) & = \left ( \nabla_{\mathbf{x}} \mathcal{L}(\mathbf{x}_t) \right)^T (\mathbf{x}_{t-1} - \mathbf{x}_{t}) + \frac{1}{2} (\mathbf{x}_{t-1} - \mathbf{x}_{t})^T H(\mathcal{L}(c)) (\mathbf{x}_{t-1} - \mathbf{x}_{t})\\
        & \leq \left ( \nabla_{\mathbf{x}} \mathcal{L}(\mathbf{x}_t) \right)^T (\mathbf{x}_{t-1} - \mathbf{x}_{t}) + \frac{\alpha^{\prime}}{2} (\mathbf{x}_{t-1} - \mathbf{x}_{t})^T (\mathbf{x}_{t-1} - \mathbf{x}_{t}).
    \end{split}
    \label{DSG-upp1}
    \end{equation}
    The inequality is because the eigenvalue of $H(\mathcal{L}(c))$ is bounded within $[ \beta^{\prime}, \alpha^{\prime} ]$. And $\mathbf{x}_{t-1} = \arg \min_{\mathbf{x}} \left ( \nabla_{\mathbf{x}} \mathcal{L}(\mathbf{x}_t) \right)^T (\mathbf{x} - \mathbf{x}_{t})$ where $\mathbf{x} \in S_{u_{\theta}, \sqrt{n} \tilde{\sigma}_t}^{n}$, it means that we choose any other value to substitute for $\mathbf{x}_{t-1}$ will amplify the upper bound in Eq.(\ref{DSG-upp1}). By replacing $\mathbf{x}_{t-1}$ with $\mathbf{x}_{t-1}^{+}$ where 

    \begin{equation}
        \mathbf{x}_{t-1}^{+} = \frac{\left ( \nabla_{\mathbf{x}} \mathcal{L}(\mathbf{x}_t) \right)^T \mathbf{x}_{t} - (\alpha^{\prime}/2)(\mathbf{x}_{t-1} - \mathbf{x}_{t})^T (\mathbf{x}_{t-1} - \mathbf{x}_{t})}{\left ( \nabla_{\mathbf{x}} \mathcal{L}(\mathbf{x}_t) \right)^T \left ( \nabla_{\mathbf{x}} \mathcal{L}(\mathbf{x}_t) \right)} \cdot \left ( \nabla_{\mathbf{x}} \mathcal{L}(\mathbf{x}_t) \right)
    \end{equation}
    , we have 
    \begin{equation}
        \mathcal{L}(\mathbf{x}_{t-1}) -  \mathcal{L}(\mathbf{x}_{t}) \leq \left ( \nabla_{\mathbf{x}} \mathcal{L}(\mathbf{x}_t) \right)^T (\mathbf{x}_{t-1}^{+} - \mathbf{x}_{t}) + \frac{\alpha^{\prime}}{2} (\mathbf{x}_{t-1} - \mathbf{x}_{t})^T (\mathbf{x}_{t-1} - \mathbf{x}_{t}) = 0.
    \end{equation}  
\end{proof}

This indicates that $\{ \mathcal{L}(\mathbf{x}_{t}) \}_{t=T}^{0}$ is a decreasing sequence. Thus, we can conclude that Algorithm \ref{alg: DSG} will decrease the attack loss when denoising steps $t$ move on. Next, we derive the lower bound of $\mathcal{L}(\mathbf{x}_{t-1}) -  \mathcal{L}(\mathbf{x}_{t})$. The lower bound indicates the decreasing speed of the attack loss. Before deriving the lower bound, we need Lemma \ref{theorem4} about the upper bound of $\| \nabla_{\mathbf{x}} \mathcal{L}(\mathbf{x}_t) \|_2$.

\begin{lemma}[Upper bound of $\| \nabla_{\mathbf{x}} \mathcal{L}(\mathbf{x}_t) \|_2$]
    Under conditions in Assumptions \ref{assu1}-\ref{assu4}, the upper bound of $\| \nabla_{\mathbf{x}} \mathcal{L}(\mathbf{x}_t) \|_2$ is:

    \begin{equation}
        \| \nabla_{\mathbf{x}} \mathcal{L}(\mathbf{x}_t) \|_2 \leq \sqrt{\lambda_{max}(\mathbf{J}_g(\mathbf{x})\mathbf{J}_g(\mathbf{x})^T)} \cdot L_g
    \end{equation}
   , where $L_g = \frac{d}{\sqrt{2 \pi \sigma^2}} \exp(-\frac{1}{2\sigma^2})$.
   \label{theorem4}
\end{lemma}

\begin{proof}
    \begin{equation}
        \begin{split}
            \| \nabla_{\mathbf{x}} \mathcal{L}(\mathbf{x}) \|_2^2 &= \left( \nabla_{\mathbf{x}} \mathcal{L}(\mathbf{x}) \right)^T \left( \nabla_{\mathbf{x}} \mathcal{L}(\mathbf{x}) \right)\\
            &= \left( \mathbf{J}_g(\mathbf{x})^T \nabla_g \mathcal{L}(g(\mathbf{x})) \right)^T \left( \mathbf{J}_g(\mathbf{x})^T \nabla_g \mathcal{L}(g(\mathbf{x})) \right)\\
            &= \nabla_g \mathcal{L}(g(\mathbf{x}))^T \left( \mathbf{J}_g(\mathbf{x})\mathbf{J}_g(\mathbf{x})^T \right)\nabla_g \mathcal{L}(g(\mathbf{x}))
        \end{split}
    \end{equation}
    Assuming the eigendecomposition of $\mathbf{J}_g(\mathbf{x})\mathbf{J}_g(\mathbf{x})^T$ is $\mathbf{J}_g(\mathbf{x})\mathbf{J}_g(\mathbf{x})^T = U^T \Lambda U$, then we have:
    \begin{equation}
        \begin{split}
            \| \nabla_{\mathbf{x}} \mathcal{L}(\mathbf{x}) \|_2^2 &= \nabla_g \mathcal{L}(g(\mathbf{x}))^T \left( U^T \Lambda U \right)\nabla_g \mathcal{L}(g(\mathbf{x})) \\
            &= \nabla_g \mathcal{L}(g(\mathbf{x}))^T \left( U^T \Lambda^{\frac{1}{2}} \Lambda^{\frac{1}{2}} U\right)\nabla_g \mathcal{L}(g(\mathbf{x}))\\
            &=\nabla_g \mathcal{L}(g(\mathbf{x}))^T \left( U^T (\Lambda^{\frac{1}{2}})^T \Lambda^{\frac{1}{2}} U\right)\nabla_g \mathcal{L}(g(\mathbf{x}))\\
            &= \left( \Lambda^{\frac{1}{2}} U \nabla_g \mathcal{L}(g(\mathbf{x})) \right)^T \left( \Lambda^{\frac{1}{2}} U \nabla_g \mathcal{L}(g(\mathbf{x})) \right)\\
            &= \| \Lambda^{\frac{1}{2}} U \nabla_g \mathcal{L}(g(\mathbf{x})) \|_2^2.
        \end{split}
    \end{equation}
    Thus, 
    \begin{equation}
        \begin{split}
            \| \nabla_{\mathbf{x}} \mathcal{L}(\mathbf{x}) \|_2 &= \| \Lambda^{\frac{1}{2}} U \nabla_g \mathcal{L}(g(\mathbf{x})) \|_2\\
            & \leq \| \Lambda^{\frac{1}{2}} \|_2 \| U \nabla_g \mathcal{L}(g(\mathbf{x})) \|_2 = \| \Lambda^{\frac{1}{2}} \|_2 \| \nabla_g \mathcal{L}(g(\mathbf{x})) \|_2 \\
            & \leq \sqrt{\lambda_{max}(\mathbf{J}_g(\mathbf{x})\mathbf{J}_g(\mathbf{x})^T)} \cdot L_g
        \end{split}
    \end{equation}
    , where $L_g = \frac{d}{\sqrt{2 \pi \sigma^2}} \exp(-\frac{1}{2\sigma^2})$.
\end{proof}

\begin{theorem}
    Under the conditions in Theorem \ref{theorem4}, the attack loss further satisfies
    \begin{equation}
        \mathcal{L}(\mathbf{x}_{t-1}) - \mathcal{L}(\mathbf{x}_{t}) \geq \frac{\beta^{\prime} n \tilde{\sigma}_t^2}{2} +\beta^{\prime} k \sqrt{n} \tilde{\sigma}_t - (k + \sqrt{n} \tilde{\sigma}_t) (\sqrt{\lambda_{max}(\mathbf{J}_g(\mathbf{x})\mathbf{J}_g(\mathbf{x})^T)} \cdot L_g) + \frac{\beta^{\prime} k^2}{2}
    \end{equation}
    , where $k$ is a positive constant.
    \label{DSGlower}
\end{theorem}

\begin{proof}

\begin{equation}
    \begin{split}
        \mathcal{L}(\mathbf{x}_{t-1}) -  \mathcal{L}(\mathbf{x}_{t}) & = \left ( \nabla_{\mathbf{x}} \mathcal{L}(\mathbf{x}_t) \right)^T (\mathbf{x}_{t-1} - \mathbf{x}_{t}) + \frac{1}{2} (\mathbf{x}_{t-1} - \mathbf{x}_{t})^T H(\mathcal{L}(c)) (\mathbf{x}_{t-1} - \mathbf{x}_{t})\\
        & \geq \underset{(*)}{\underbrace{\left ( \nabla_{\mathbf{x}} \mathcal{L}(\mathbf{x}_t) \right)^T (\mathbf{x}_{t-1} - \mathbf{x}_{t})}} + \underset{(**)}{\underbrace{\frac{\beta^{\prime}}{2} (\mathbf{x}_{t-1} - \mathbf{x}_{t})^T (\mathbf{x}_{t-1} - \mathbf{x}_{t})}}
    \end{split}
    \label{DSG-low1}
    \end{equation}

In Eq.(\ref{DSG-low1}), $\mathbf{x}_{t-1} = \arg \min_{\mathbf{x}} \left ( \nabla_{\mathbf{x}} \mathcal{L}(\mathbf{x}_t) \right)^T (\mathbf{x} - \mathbf{x}_{t})$ where $\mathbf{x} \in S_{u_{\theta}, \sqrt{n} \tilde{\sigma}_t}^{n}$. In this case, we have $\mathbf{x}_{t-1} = u_{\theta}(\mathbf{x}_{t}) - \sqrt{n} \tilde{\sigma}_t d$ where $d = \nabla_{\mathbf{x}} \mathcal{L}(\mathbf{x}_t) / \| \nabla_{\mathbf{x}} \mathcal{L}(\mathbf{x}_t) \|_2$. For $(*)$, we have 

\begin{equation}
\begin{split}
    (*)  &= \left ( \nabla_{\mathbf{x}} \mathcal{L}(\mathbf{x}_t) \right)^T (\mathbf{x}_{t-1}  - u_{\theta}(\mathbf{x}_{t}) + u_{\theta}(\mathbf{x}_{t}) - \mathbf{x}_{t})\\
    &= -\sqrt{n} \tilde{\sigma}_t \| \nabla_{\mathbf{x}} \mathcal{L}(\mathbf{x}_t) \|_2 + \left ( \nabla_{\mathbf{x}} \mathcal{L}(\mathbf{x}_t) \right)^T \left(u_{\theta}(\mathbf{x}_{t}) - \mathbf{x}_{t}\right).
\end{split}
\end{equation}

For $(**)$, we have

\begin{equation}
    \begin{split}
        (**) &= \frac{\beta^{\prime}}{2} (\mathbf{x}_{t-1} - \mathbf{x}_{t})^T (\mathbf{x}_{t-1} - \mathbf{x}_{t}) = \frac{\beta^{\prime}}{2} (\mathbf{x}_{t-1} - u_{\theta}(\mathbf{x}_{t}))^T (\mathbf{x}_{t-1} - u_{\theta}(\mathbf{x}_{t})) \\
        & \quad +\frac{\beta^{\prime}}{2} (u_{\theta}(\mathbf{x}_{t})-\mathbf{x}_{t})^T (u_{\theta}(\mathbf{x}_{t})-\mathbf{x}_{t}) + \mu (\mathbf{x}_{t-1} - u_{\theta}(\mathbf{x}_{t}))^T \left(u_{\theta}(\mathbf{x}_{t}) - \mathbf{x}_{t}\right)\\
        &=\frac{\beta^{\prime} n \tilde{\sigma}_t^2}{2} + \frac{\beta^{\prime}}{2} \left ( 2\mathbf{x}_{t-1} - \mathbf{x}_{t} - u_{\theta}(\mathbf{x}_{t})\right)^T\left(u_{\theta}(\mathbf{x}_{t}) - \mathbf{x}_{t}\right)\\
    &= \frac{\beta^{\prime} n \tilde{\sigma}_t^2}{2} + \frac{\beta^{\prime}}{2} \left ( \mathbf{x}_{t-1} - \mathbf{x}_{t} \right)^T\left(u_{\theta}(\mathbf{x}_{t}) - \mathbf{x}_{t}\right) + \frac{\beta^{\prime}}{2} \left ( \mathbf{x}_{t-1} - u_{\theta}(\mathbf{x}_{t})\right)^T\left(u_{\theta}(\mathbf{x}_{t}) - \mathbf{x}_{t}\right).
    \end{split}
    \label{(**)}
\end{equation}
    To derive the lower bound of $\mathcal{L}(\mathbf{x}_{t-1}) -  \mathcal{L}(\mathbf{x}_{t})$, the common factor is $u_{\theta}(\mathbf{x}_{t}) - \mathbf{x}_{t}$. The part $\frac{\beta^{\prime}}{2} \left ( \mathbf{x}_{t-1} - \mathbf{x}_{t} \right)^T\left(u_{\theta}(\mathbf{x}_{t}) - \mathbf{x}_{t}\right)$ and part $\frac{\beta^{\prime}}{2} \left ( \mathbf{x}_{t-1} - u_{\theta}(\mathbf{x}_{t})\right)^T\left(u_{\theta}(\mathbf{x}_{t}) - \mathbf{x}_{t}\right)$ in Eq.(\ref{(**)}) are two inner products. To explain how the denoising ability can reduce the attack loss, we want to find a value that can be actually reached. Noting the randomness in $u_{\theta}(\mathbf{x}_{t})$, by setting $u_{\theta}(\mathbf{x}_{t}) - \mathbf{x}_{t} = -kd (k>0)$, the part $\frac{\beta^{\prime}}{2} \left ( \mathbf{x}_{t-1} - u_{\theta}(\mathbf{x}_{t})\right)^T\left(u_{\theta}(\mathbf{x}_{t}) - \mathbf{x}_{t}\right)$ in Eq.(\ref{(**)}) can be expressed :
    \begin{equation}
    \begin{split}
        \frac{\beta^{\prime}}{2} \left ( \mathbf{x}_{t-1} - u_{\theta}(\mathbf{x}_{t})\right)^T\left(u_{\theta}(\mathbf{x}_{t}) - \mathbf{x}_{t}\right) &= - \frac{\beta^{\prime}}{2} (\sqrt{n} \tilde{\sigma}_t d)^T (-kd)\\
        & = \frac{\beta^{\prime} k \sqrt{n} \tilde{\sigma}_t }{2}.
    \end{split}  
    \end{equation}

    Then the part $\frac{\beta^{\prime}}{2} \left ( \mathbf{x}_{t-1} - \mathbf{x}_{t} \right)^T\left(u_{\theta}(\mathbf{x}_{t}) - \mathbf{x}_{t}\right)$ in Eq.(\ref{(**)}) is 
    \begin{equation}
        \begin{split}
            \frac{\beta^{\prime}}{2} \left ( \mathbf{x}_{t-1} - \mathbf{x}_{t} \right)^T\left(u_{\theta}(\mathbf{x}_{t}) - \mathbf{x}_{t}\right) &= \frac{\beta^{\prime}}{2} \left ( \mathbf{x}_{t-1} - u_{\theta}(\mathbf{x}_{t}) -kd \right)^T\left(-kd\right)\\ & = - \frac{\beta^{\prime}}{2} \left( (\sqrt{n} \tilde{\sigma}_t - k) d \right)^T(-kd) = \frac{\beta^{\prime} k(\sqrt{n} \tilde{\sigma}_t + k)}{2}. 
        \end{split}
    \end{equation}

The $(*)$ is 
\begin{equation}
\begin{split}
    (*) & = -\sqrt{n} \tilde{\sigma}_t \| \nabla_{\mathbf{x}} \mathcal{L}(\mathbf{x}_t) \|_2 + \left ( \nabla_{\mathbf{x}} \mathcal{L}(\mathbf{x}_t) \right)^T \left(u_{\theta}(\mathbf{x}_{t}) - \mathbf{x}_{t}\right)\\
    &= -\sqrt{n} \tilde{\sigma}_t \| \nabla_{\mathbf{x}} \mathcal{L}(\mathbf{x}_t) \|_2 - k \| \nabla_{\mathbf{x}} \mathcal{L}(\mathbf{x}_t) \|_2.
\end{split}
\end{equation}
    Finally, we have
\begin{equation}
    \begin{split}
        \mathcal{L}(\mathbf{x}_{t-1}) -  \mathcal{L}(\mathbf{x}_{t}) & \geq \frac{\beta^{\prime} n \tilde{\sigma}_t^2}{2} +\beta^{\prime} k \sqrt{n} \tilde{\sigma}_t - (k + \sqrt{n} \tilde{\sigma}_t) \| \nabla_{\mathbf{x}} \mathcal{L}(\mathbf{x}_t) \|_2 + \frac{\beta^{\prime} k^2}{2}\\
        & \geq \frac{\beta^{\prime} n \tilde{\sigma}_t^2}{2} +\beta^{\prime} k \sqrt{n} \tilde{\sigma}_t - (k + \sqrt{n} \tilde{\sigma}_t) (\sqrt{\lambda_{max}(\mathbf{J}_g(\mathbf{x})\mathbf{J}_g(\mathbf{x})^T)} \cdot L_g) + \frac{\beta^{\prime} k^2}{2}.
    \end{split}
    \label{DSG-low2}
    \end{equation}
\end{proof}

To better understand the proof, see Figure \ref{noise_lap_appendix} for details. If $u_{\theta}(\mathbf{x}_{t}) - \mathbf{x}_{t} = -kd (k>0)$ is not satisfied, it is the case in Figure \ref{DSG_01}. Due to the randomness of $\mathbf{x}_{t-1}$, $u_{\theta}(\mathbf{x}_{t}) - \mathbf{x}_{t} = -kd (k>0)$ can be satisfied which is the case in Figure \ref{DSG_02}. In this case, $\mathcal{L}(\mathbf{x}_{t-1}) -  \mathcal{L}(\mathbf{x}_{t})$ reaches its lower bound.

Analogous to the discussion about Theorem \ref{lowerboundDSG}, for Eq.(\ref{DSG-low2}), the information loss associated with the mapping $\mathbf{x} \mapsto g$ is expressed by the term $\sqrt{\lambda_{max}(\mathbf{J}_g(\mathbf{x})\mathbf{J}_g(\mathbf{x})^T)}$. The information loss is also scaled by $\tilde{\sigma}_t$, which is a coefficient in the denoising steps of diffusion models. From the fact $\tilde{\sigma}_t \cdot L_g \leq \frac{d \tilde{\sigma}_t}{\sqrt{2 \pi \sigma^2}}$, for a large differentially private noise scale $\sigma^2$, the term $\tilde{\sigma}_t$ can be chosen to be sufficiently large, ensuring that $\frac{d \cdot \tilde{\sigma}_t}{\sqrt{2 \pi \sigma^2}}$ remains large, which indicates the small lower bound of $\mathcal{L}(\mathbf{x}_{t-1}) - \mathcal{L}(\mathbf{x}_{t})$.

\begin{figure*}[htp!]
    \centering
    \vspace{-2mm}
    \subfigure[$u_{\theta}(\mathbf{x}_{t}) - \mathbf{x}_{t} \neq -kd (k>0)$]{
		\includegraphics[width=0.45\textwidth]{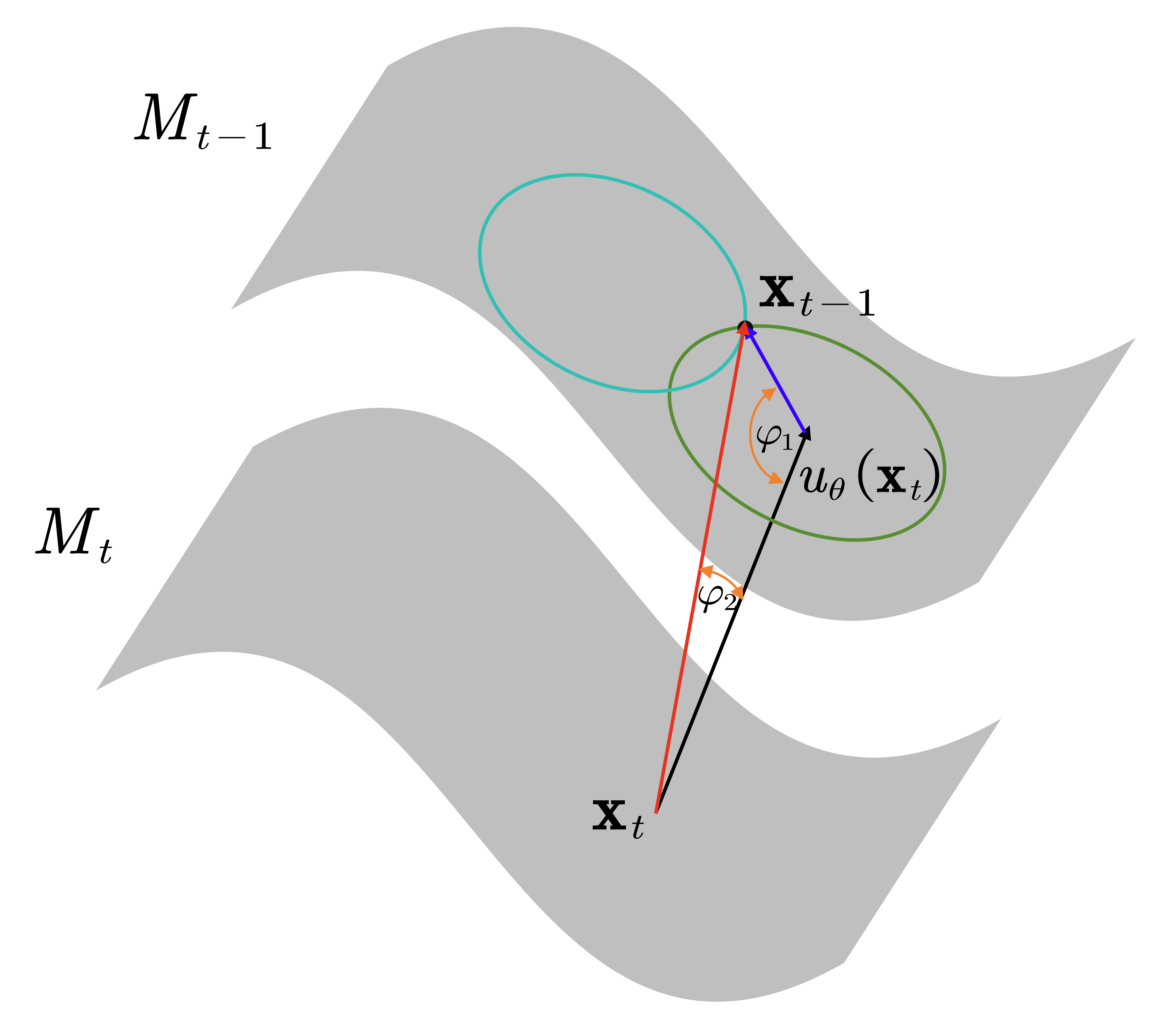}
		\label{DSG_01}
    }
    \subfigure[$u_{\theta}(\mathbf{x}_{t}) - \mathbf{x}_{t} = -kd (k>0)$]{
		\includegraphics[width=0.45\textwidth]{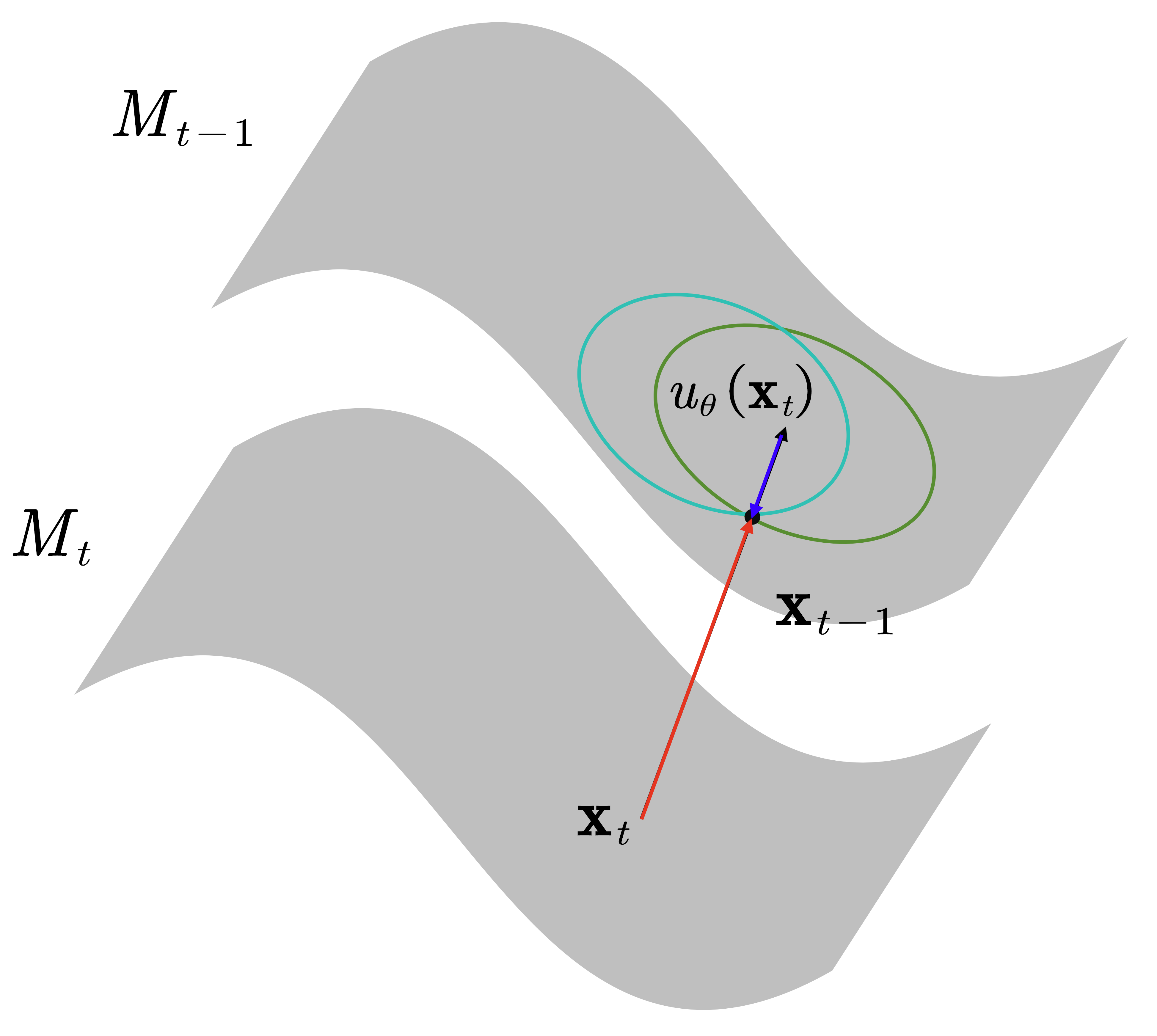}
		\label{DSG_02}
    }
    \caption{Explanations for Theorem \ref{DSGlower}}
    \label{noise_lap_appendix}
\end{figure*}

\section{Experiments}

\subsection{Experimental Setup}
\textbf{Machine Configuration.} All experiments are run over a GPU machine with one Intel(R) Xeon(R) Gold 5218R CPU @ 2.10GHz with 251 GB DRAM and 4 NVIDIA RTX A6000 GPUs (each GPU has 48 GB DRAM). We implement our attack with Pytorch 2.0.1.

\textbf{Datasets and Pre-trained Diffusion Model.} We select CelebA-HQ \cite{celeba} as experimental dataset, the resolution of images in which is $256 \times 256$ pixels. The pre-trained diffusion model used in our methods is taken from \cite{pre-trained_diffusion}, the resolution of inputs of which is $256 \times 256$ pixels.

\textbf{The Attacked Models.} In the experiment, we use several models under attack, including MLP\_1, MLP\_2, MLP\_3, CNN and ResNet9. The resolution of the input images for these models are $256 \times 256$ pixels. And the architectures of them are as follows:

\textit{MLP\_1.} MLP\_1 consists of three fully connected layers. The input is flattened to $3 \times 256 \times 256$ and passed through a 512-unit layer, followed by ReLU activation. The second layer has 128 units with ReLU activation, and the output layer has 2 units for classification.

\textit{MLP\_2.} MLP\_2 has four fully connected layers. The input is flattened and passed through layers of 512, 256, and 64 units, each followed by ReLU activation. The final output layer has 2 units for classification.

\textit{MLP\_3.} MLP\_3 consists of five fully connected layers. The input is flattened and passed through layers of 1028, 512, 256, and 64 units, each followed by ReLU activation. The output layer has 2 units for classification.

\textit{CNN.} The CNN architecture consists of two convolutional layers. The first layer has 64 filters, and the second has 128 filters, both with $3 \times 3$ kernels, followed by ReLU activation and 2x2 MaxPooling for downsampling. After flattening, a fully connected layer with 256 units and ReLU activation is applied, followed by a final output layer with 2 units and softmax for classification. In our experiment, the default attacked model is CNN.

\textit{ResNet9.} ResNet9 is a simplified version of the standard ResNet, consisting of two convolutional layers, two residual blocks, and two additional convolutional layers. It reduces the number of residual blocks compared to the standard ResNet and uses fewer convolutional layers for feature extraction. The model ends with a global average pooling layer and a fully connected layer for classification.

\textbf{Evaluation Metrics.} We adopt \textit{MSE (Mean Square Error)}, \textit{SSIM (Structural Similarity Index)}, \textit{PSNR (Peak Signal-to-Noise Ratio)}, and \textit{LPIPS (Learned Perceptual Image Patch Similarity)} to measure the quality of the reconstructed images, which are commonly used in computer vision.

\textit{MSE.} MSE is used to quantify the difference between a target image and the relevant reconstructed image. A lower MSE value indicates that the reconstructed image is closer to the target image, implying better reconstruction quality. 

\textit{SSIM.} SSIM is also used for measuring the similarity between two images, which is considered to be more perceptually relevant than traditional methods like MSE since it incorporates changes in structural information, luminance, and contrast. SSIM ranges from -1 to 1. A value of 1 indicates perfect similarity, and it is desirable to have a higher SSIM when evaluating reconstruction quality.

\textit{PSNR.} PSNR is a popular metric used the assessment of image reconstruction quality. It measures the ratio between the maximum possible power of a signal (in our experiments, the target image) and the power of distorting noise that affects its representation (in our experiments, the reconstructed image). A higher PSNR value implies that the reconstruction is of higher quality. In contrast, a lower PSNR means more error or noise.

\textit{LPIPS.} Unlike SSIM and MSE, which are handcrafted metrics, LPIPS leverages deep learning to more closely align with human perceptual judgments. LPIPS measures perceptual similarity by extracting deep features from images using a pre-trained convolutional neural network (CNN) and calculating the distance (typically Euclidean) among these feature representations. A lower LPIPS score denotes higher similarity and better reconstruction quality.


\subsection{Results and Discussion}
\subsubsection{Reconstructions Without Differentially Private Noise}
In this subsection, we present our experiment results when the attacker captures the original gradient. The original gradient is also denoted as the target gradient from which the attacker aims to extract the corresponding private image. In the following parts, $\hat{\mathbf{x}}_0(\mathbf{x}_t)$ represents the prediction of $x_0$ which is inferred via timestep $t$ and the output $\mathbf{x}_t$ of the reverse process at timestep $t$.

Figure \ref{reconstruction_x} presents the Euclidean distance between the gradient of $\hat{\mathbf{x}}_0(\mathbf{x}_t)$ and the target gradient where $\hat{\mathbf{x}}_0(\mathbf{x}_t)$ is the reconstructed image derived from the conditional diffusion attack model, and the MSE distance between $\hat{\mathbf{x}}_0(\mathbf{x}_t)$ and the target private image. Additionally, images $\hat{\mathbf{x}}_0(\mathbf{x}_t)$ and the target private image are also displayed. Figure \ref{x_process} shows the variation process of images $x_{t}$ from the initial stage to the final result during the reverse process of the conditional diffusion attack model. From Figure \ref{reconstruction_x} and Figure \ref{x_process}, the proposed two gradient-guided methods are extremely effective in extracting the private image from the original gradient. As the gradient's distance decreases, the MSE distance between the reconstructed image and the target private image also decreases and the reconstructed image gradually approaches the private image. 

\begin{figure*}[htp!]
    \centering
    \vspace{-2mm}
    \subfigure[Reconstruction Process of DPS-based Method]{
		\includegraphics[width=0.48\textwidth]{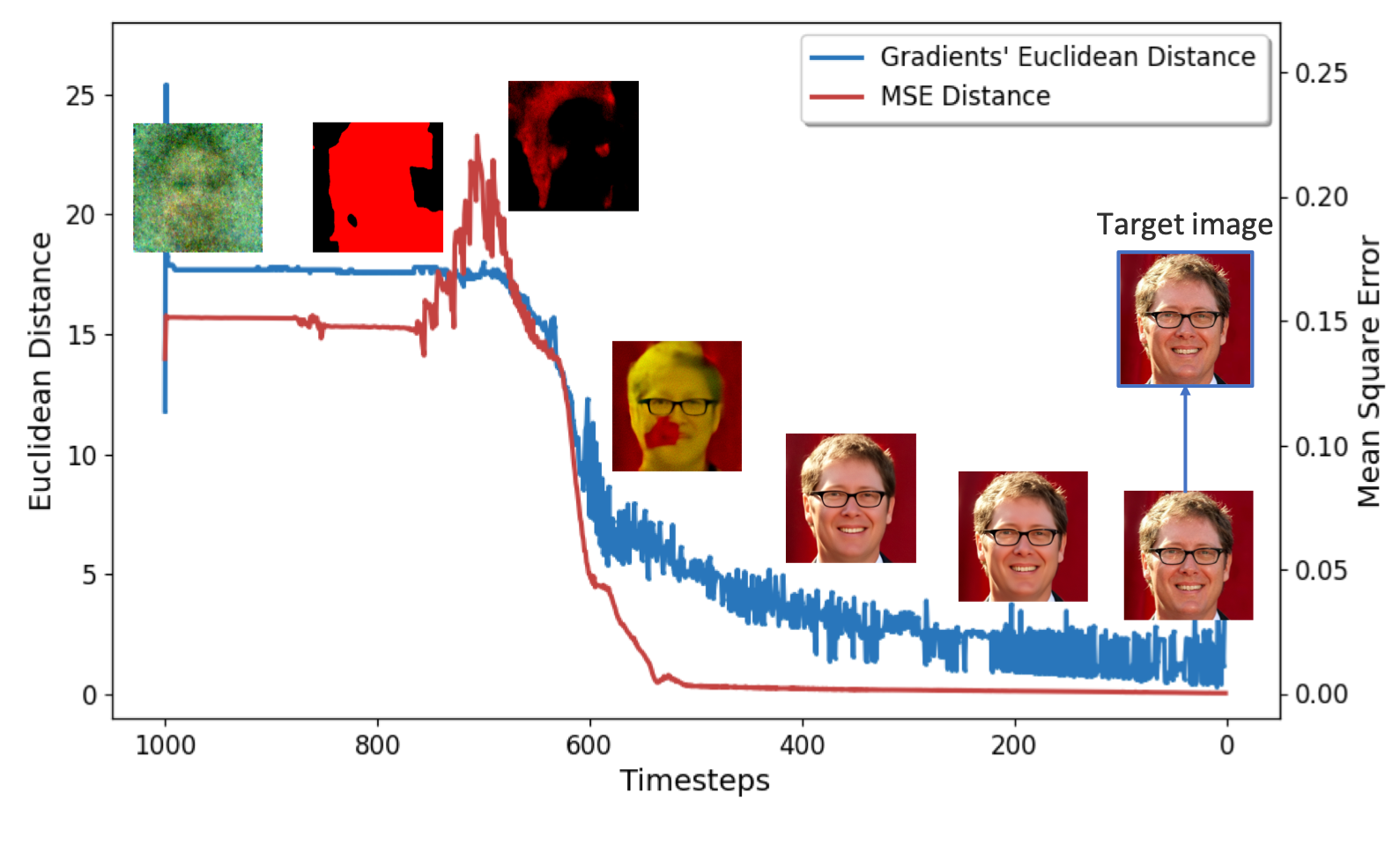}
		\label{reconstruction_dps_x0}
    }
    \subfigure[Reconstruction Process of DSG-based Method]{
		\includegraphics[width=0.48\textwidth]{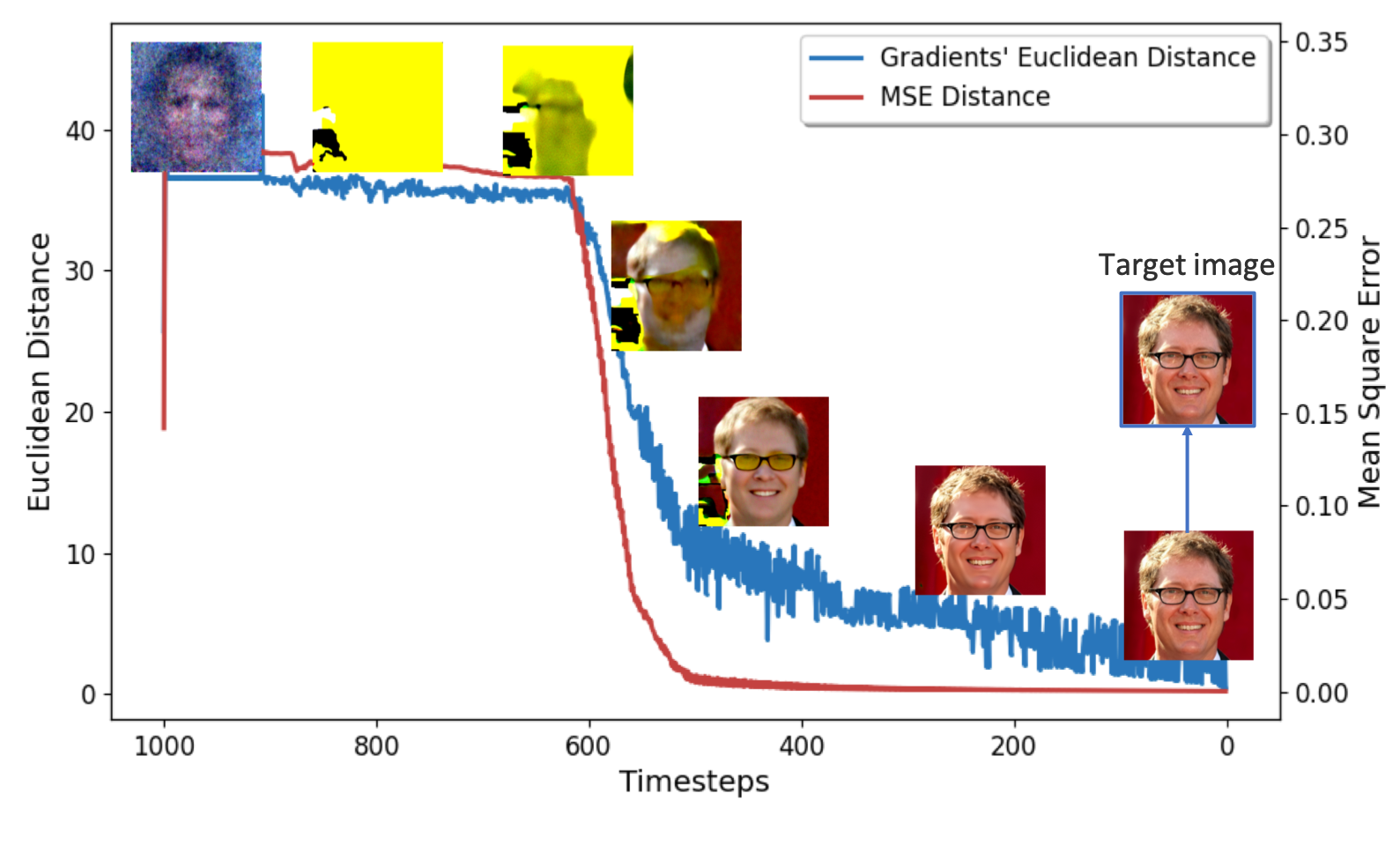}
		\label{reconstruction_dsg_x0}
    }
    \caption{Reconstruction Processes}
    \label{reconstruction_x}
\end{figure*}

\begin{figure*}[htbp]
  \centering
  \includegraphics[scale=0.5]{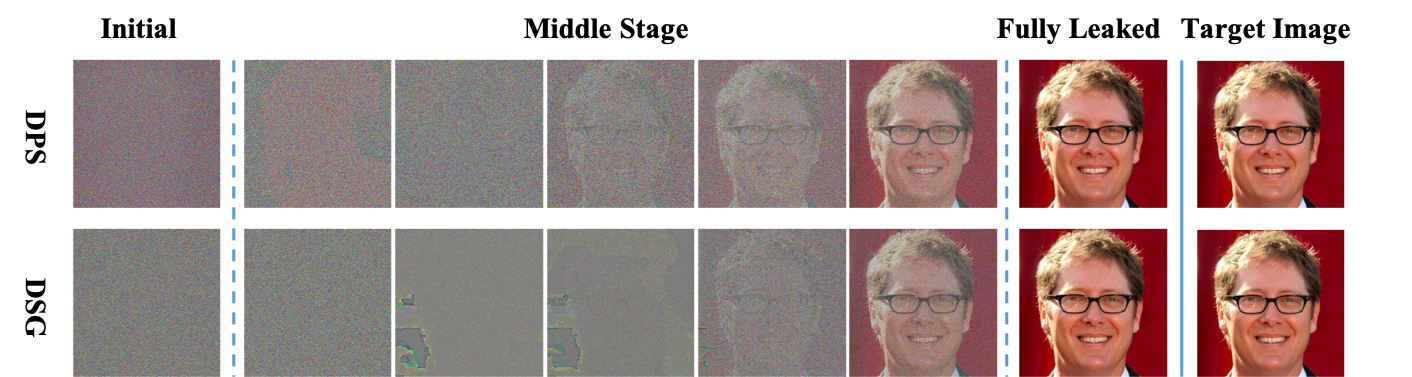}
  \caption{Images $x_{t}$ of Reconstruction Process}
  \label{x_process}
\end{figure*}

\subsubsection{Impacts of Gaussian noise scale \texorpdfstring{$\sigma^2$}{sigma2}}

\textbf{Noisy Gradients.} Previous studies on differential privacy have explored adding Gaussian noise directly to gradients to mitigate image reconstruction attacks. To evaluate the robustness of our attack methods under similar conditions, we adopt this approach by perturbing the original gradients with Gaussian noise. The adversary receives the perturbed gradient $g^{\prime} = g + \mathcal{N}(0,\sigma^2)$, where $\mathcal{N}(0,\sigma^2)$ represents Gaussian noise with variance $\sigma^2$. Our experiments aim to determine whether the added noise effectively impedes our attack methods. 

To quantitatively assess the impact of Gaussian noise on reconstruction performance, we use the Peak Signal-to-Noise Ratio (PSNR) metric, which is sensitive to differences in image quality caused by noise, provides a clear measure of reconstruction fidelity, and exhibits relatively large numerical variation across different experimental results, making it easier to observe. Figure \ref{do_not_normalize_noise} illustrates the PSNR values of the reconstructed images $\hat{\mathbf{x}}_0(\mathbf{x}_t)$ compared to the target images over varying noise magnitudes $\sigma^2$. Additionally, we not only provide the images of $x_t$ at the end of the reverse processes, but we also display the visual examples of the predicted images $\hat{\mathbf{x}}_0(\mathbf{x}_t)$ at the points where the highest PSNR is achieved during the reverse denoising process, highlighting the peak performance.

Figure \ref{do_not_normalize_noise} demonstrates that our methods remain effective at reconstructing images nearly identical to the original private images when the added Gaussian noise has a small variance ($\sigma^2$ is small). However, as the noise magnitude increases, there is a noticeable degradation in the quality of the reconstructed images. This degradation indicates that higher levels of noise impair the attacker's ability to accurately reconstruct the original image, aligning with the theoretical analysis in Section \ref{sec:4.2}, which predicts that increased noise introduces uncertainty hindering precise reconstruction.


\begin{figure*}[htp!]
    \centering
    \vspace{-2mm}
    \subfigure[DPS-based Method]{
		\includegraphics[width=0.48\textwidth]{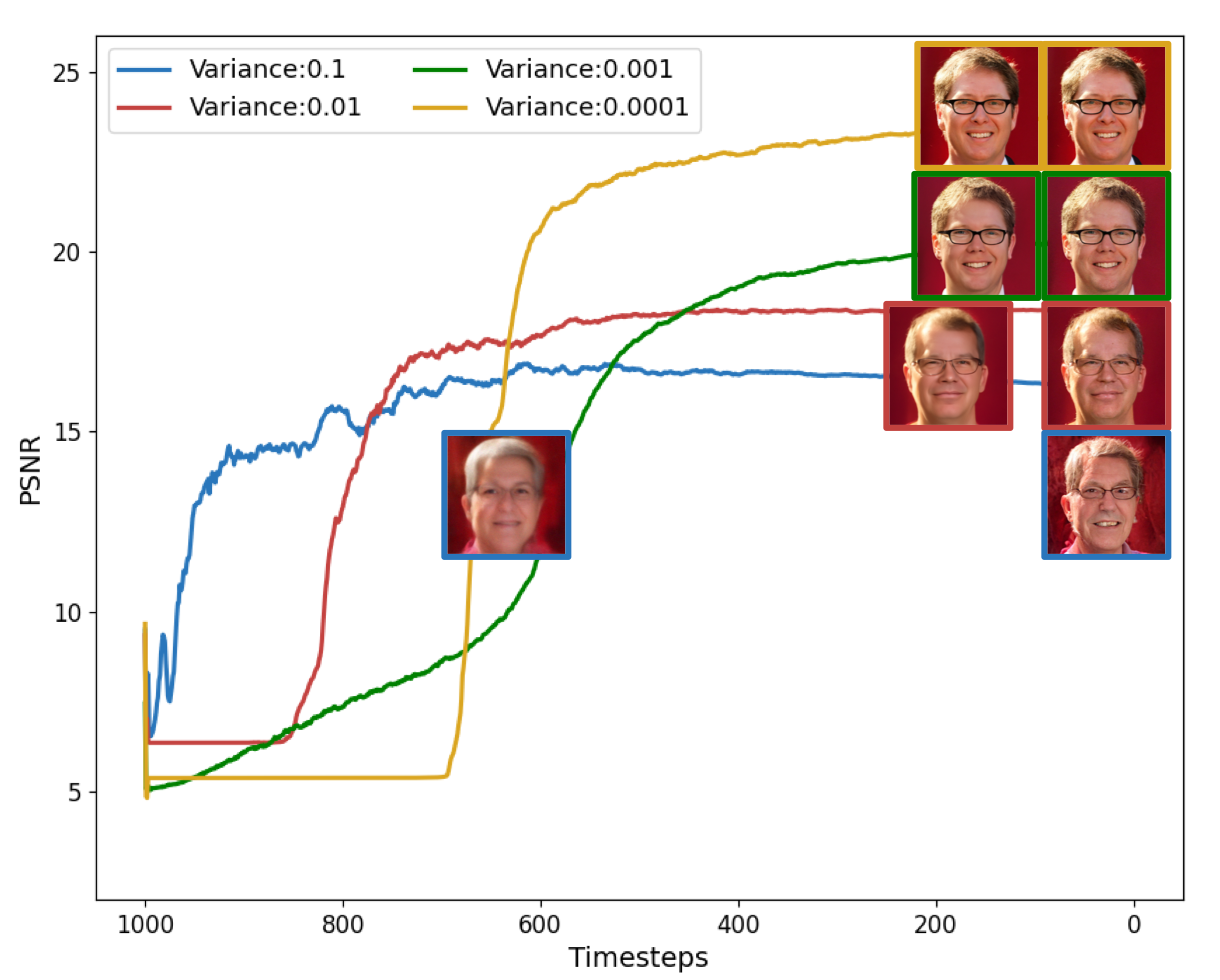}
		\label{do_not_normalize_noise_DPS}
    }
    \subfigure[DSG-based Method]{
		\includegraphics[width=0.48\textwidth]{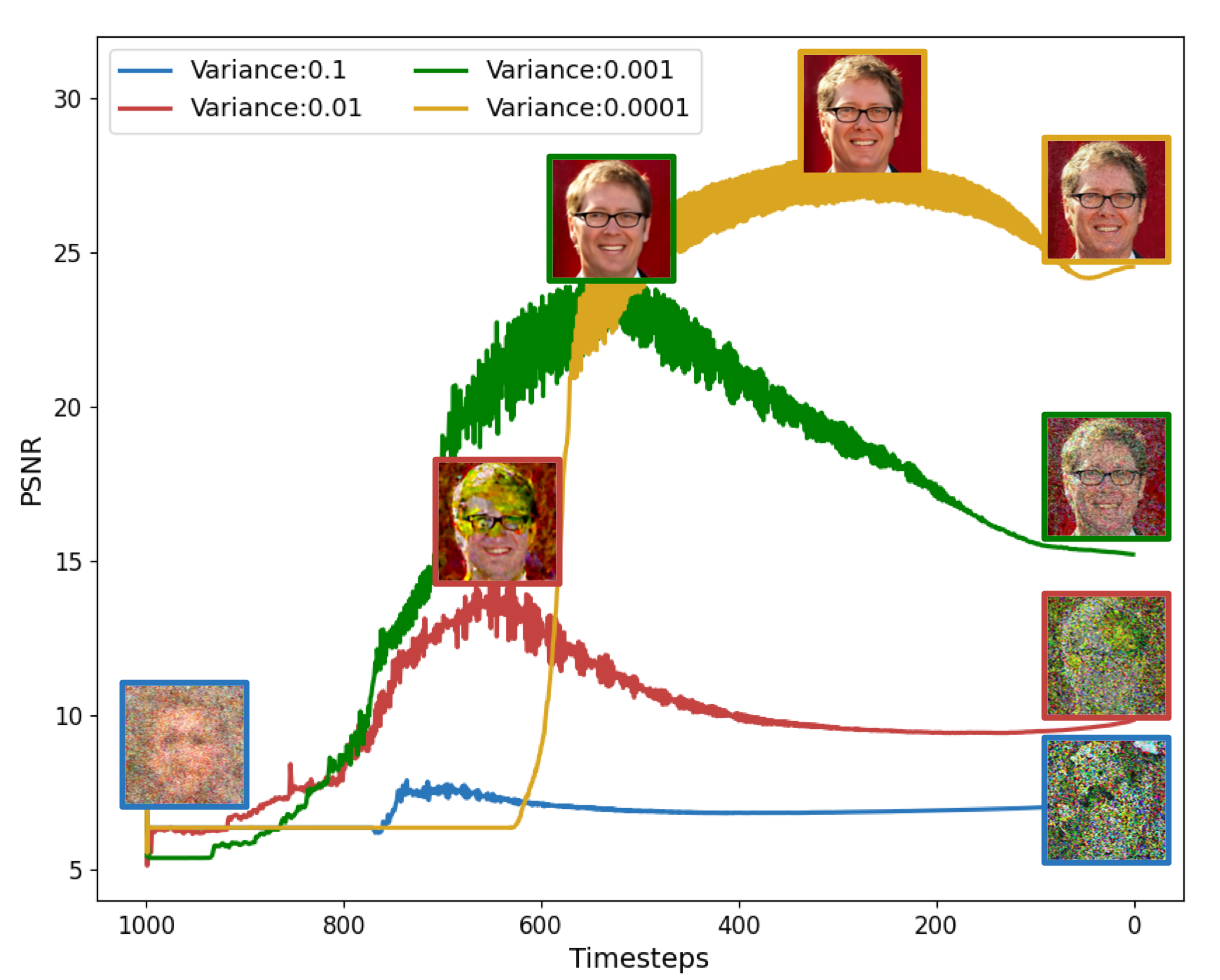}
		\label{do_not_normalize_noise_DSG}
    }
    \caption{Reconstruction with noisy gradients.}
    \label{do_not_normalize_noise}
\end{figure*}

\textbf{Noisy Normalized Gradients.} In the previous experiments, adding Gaussian noise directly to the gradients has not considered the gradients' sensitivity or the formal privacy guarantees provided by differential privacy. This omission prevents us from relating the noise variance directly to differential privacy parameters $(\varepsilon, \delta)$ since the sensitivity of the original gradients is unknown. To address this limitation, we now evaluate the impact of differential privacy-compliant noise addition on our attack methods.

We first normalize the original gradients to have a bounded sensitivity, specifically restricting their sensitivity to 1. This normalization ensures that the gradients meet the requirements for applying differential privacy mechanisms. We add Gaussian noise calibrated according to the differential privacy parameters. The adversary receives perturbed gradients computed as: $g^{\prime} = \text{Normalize}(g) + \mathcal{N}(0,\sigma^2)$.

By varying the privacy parameters $(\varepsilon, \delta)$, we can directly relate the noise magnitude to the desired privacy level and observe how different levels of privacy protection affect the performance of our reconstruction attacks. The reconstruction performance under this setting is presented in Figure \ref{normalize_noise1}. Figure \ref{normalize_noise1} illustrates that as the differential privacy parameter $\varepsilon$ decreases (implying stronger privacy guarantees), the reconstruction performance of both attack methods declines. This trend is consistent with our theoretical expectations in Section \ref{sec:4.2}, where increased noise (necessary for stronger privacy) obscures more information from the gradients, making accurate reconstruction more challenging.

\begin{figure*}[htp!]
    \centering
    \vspace{-2mm}
    \subfigure[DPS-based Method]{
		\includegraphics[width=0.48\textwidth]{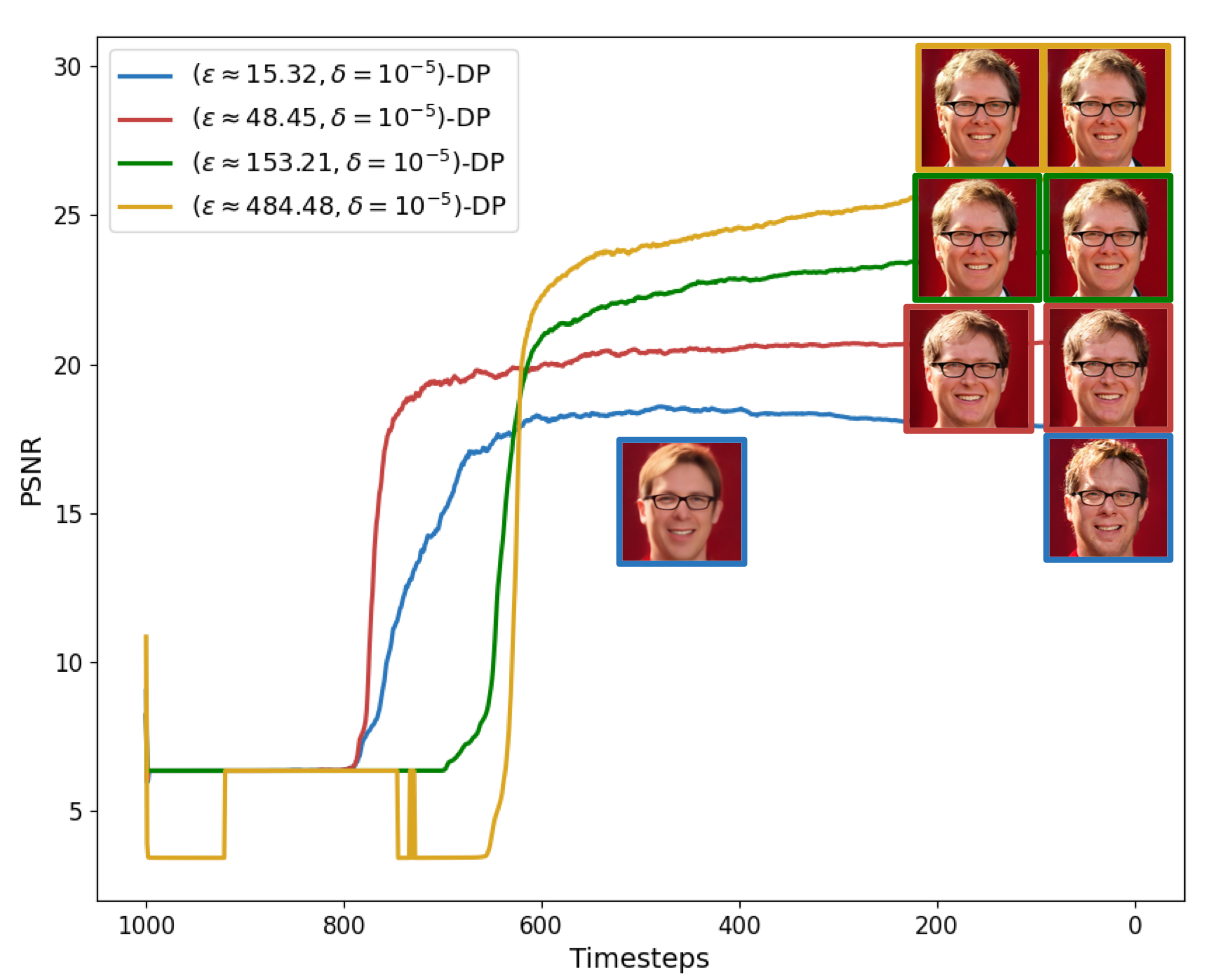}
		\label{normalize_noise_DPS}
    }
    \subfigure[DSG-based Method]{
		\includegraphics[width=0.48\textwidth]{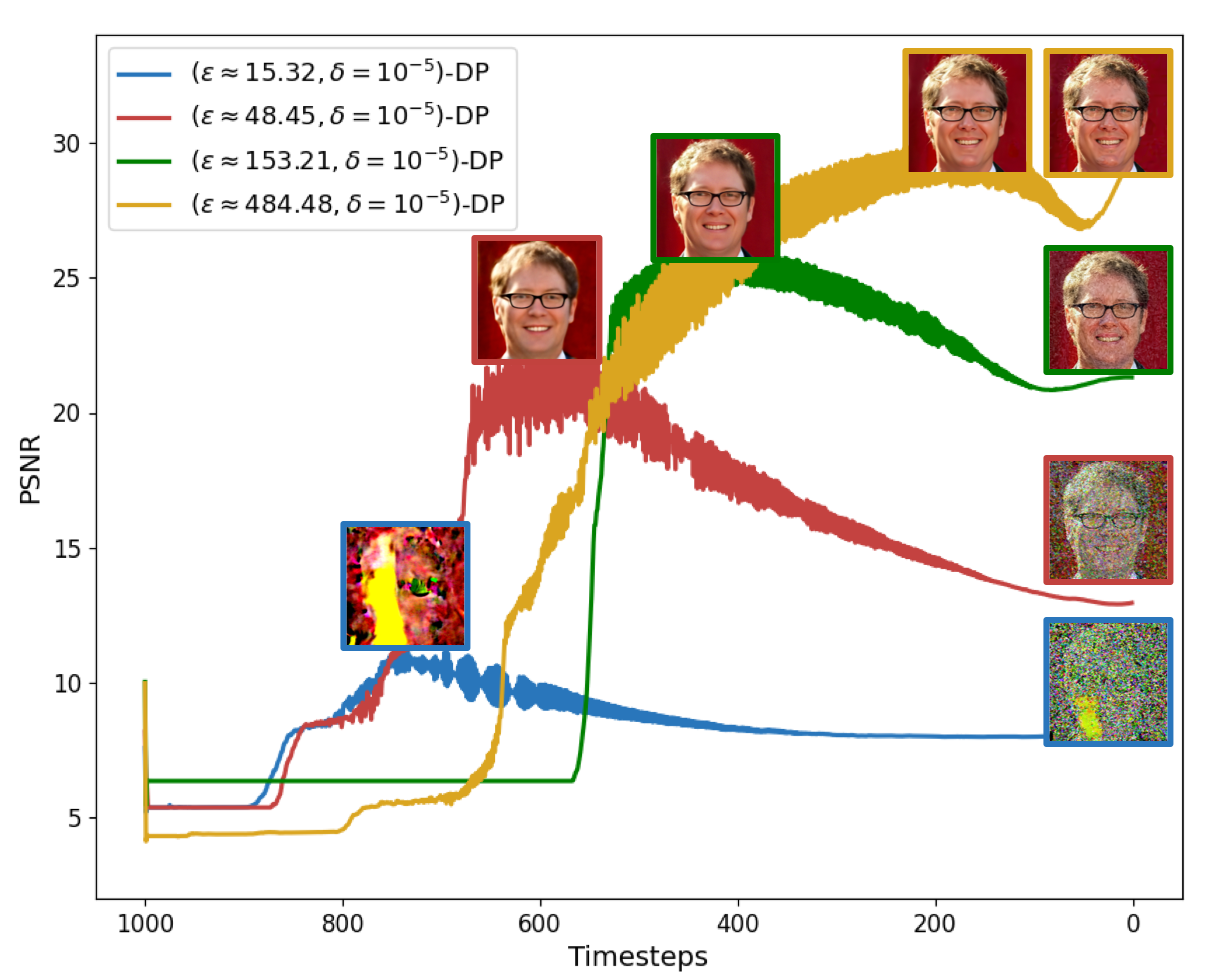}
		\label{normalize_noise_DSG}
    }
    \caption{Reconstruction with noisy normalized gradients.}
    \label{normalize_noise1}
\end{figure*}

Our analysis underscores the importance of carefully calibrating noise addition in accordance with differential privacy principles to protect against sophisticated reconstruction attacks. While increasing the noise magnitude (or decreasing $\varepsilon$) can degrade reconstruction performance, our attack methods can still extract significant information under certain conditions. Therefore, achieving a balance between model utility and privacy protection remains a critical challenge that warrants further research and optimization.

And we observe distinct behaviors between the DPS-based method and the DSG-based method: 
\begin{itemize}
    \item \textbf{DPS-based Method}: The DPS-based method's highest reconstruction quality typically appears near the final stages of the denoising process. Under strong privacy settings (small $\varepsilon$), the DPS method produces reconstructions that bear little resemblance to the original private images, only non-detailed and limited features can be reflected, such as a man with golden hair, wearing black glasses, and smiling in the image. This results in images that lack detailed features of the private data but appear less noisy overall.
    
    \item \textbf{DSG-based Method}: The DSG-based method often reaches its peak reconstruction quality during the middle stages of the reverse denoising process. As the denoising continues, the reconstructed images become progressively blurrier and noisier, suggesting that additional denoising does not improve—and may even degrade—the reconstruction quality. This "over-denoising" effect implies that the method begins to remove not only noise but also essential features of the image. However interestingly, despite the heavy noise, the original image is obscured by noise, and the faint outlines or structural features of the original images can still be roughly discerned.
\end{itemize}

These differing performances suggest that the DSG-based method retains some information about the original data even under significant noise perturbations. In contrast,the DPS-based method's reconstructions become more smooth and diverse, but less informative about the original image. The earlier peak performance of the DSG-based method suggests that an adversary might obtain better reconstructions by observing intermediate outputs rather than waiting for the final result. The performance of these two methods in reconstruction results is understandable. In DPS-based method, the reverse sampling values follow a Gaussian distribution, resulting in diverse and smooth reconstructions. However, in DSG-based method, the reverse process is deterministic, leading to many noise artifacts in the reconstruction at the end of the process (due to the addition of noise to the gradients), while still preserving the contours of the target image. The differing behaviors between the DSG-based method and the DPS-based method in terms of when they achieve optimal reconstruction highlight fundamental differences in how they utilize gradient information during the denoising process. Understanding these differences could provide valuable insights for developing more robust defenses against such attacks and present an intriguing direction for future research.

Moremore, the quantitative results in Table \ref{defense1} show that our attack methods yield higher-quality reconstructed images when using perturbed gradients obtained by adding noise after normalization—a common practice in differentially private gradient computations. This finding has significant implications for real-world applications:

\textit{Effectiveness of Current Practices}: Simply applying differential privacy mechanisms by adding noise to normalized gradients may not be sufficient to prevent reconstruction attacks, especially under weaker privacy settings (larger $\varepsilon$) where the noise may not effectively obfuscate gradient information.

\textit{Need for Stronger Privacy Guarantees}: Our experiments highlight the necessity of setting privacy parameters to enforce strong privacy guarantees. Under weaker settings, adversaries can still reconstruct high-quality images, compromising data confidentiality.

\textit{Potential for Additional Protections}: Practitioners should consider enhancing privacy protection by exploring additional or alternative techniques beyond standard differential privacy mechanisms. This includes combining multiple defense strategies or developing new methods specifically designed to counter advanced reconstruction attacks like those proposed in this work.

\begin{table*}[!htb]
    \centering
    \small
    \begin{tabular}{ccccccc}
    \hline
    \toprule
    
      & \textbf{Method}& \textbf{Variance} & \textbf{MSE$\downarrow$} & \textbf{SSIM$\uparrow$} & \textbf{PSNR$\uparrow$} & \textbf{LPIPS$\downarrow$} \\
    \midrule
            &  & $10^{-4}$ & 3.9420e-03  & 0.99977797 & 24.0428 & 1.8718e-05\\
            & DPS- & $10^{-3}$ &9.1194e-03 & 0.99924773 & 20.4003 & 4.4162e-05 \\
            & based & $10^{-2}$  & 1.4431e-02 & 0.99855620 & 18.4069 & 1.2940e-04\\
    Not  &  & $10^{-1}$ & 2.0473e-02 & 0.99773788 & 16.8882 & 2.0306e-04\\ \cline{2-7}
Normalize  & & $10^{-4}$ & 1.5746e-03 & 0.99993891 & 28.0284& 7.9936e-06\\
            & DSG- & $10^{-3}$ & 3.9366e-03&0.99972790&24.0488 & 4.5209e-05 \\
            & based & $10^{-2}$  &3.6552e-02 &0.99550867&14.9709 & 2.4720e-04\\
            &  & $10^{-1}$  &1.2520e-01 &0.98675776& 9.0239& 7.1221e-04\\ \midrule
    \multirow{16}{*}{Normalize} & & $10^{-4}$  & \multirow{2}{*}{1.9708e-03} & \multirow{2}{*}{0.99992925} & \multirow{2}{*}{27.0536} & \multirow{2}{*}{7.4911e-06}\\
    && $(484.48, 10^{-5})$-DP &&&&\\
            && $10^{-3}$ & \multirow{2}{*}{3.9300e-03} & \multirow{2}{*}{0.99977869}  & \multirow{2}{*}{24.0561}& \multirow{2}{*}{1.7936e-05}\\
    &DPS-& $(153.21, 10^{-5})$-DP &&&&\\     
            &based& $10^{-2}$ &  \multirow{2}{*}{8.3796e-03} & \multirow{2}{*}{0.99938870} & \multirow{2}{*}{20.7678} & \multirow{2}{*}{5.4881e-05}\\
    && $(48.45, 10^{-5})$-DP &&&&\\
            && $10^{-1}$ & \multirow{2}{*}{1.3867e-02} & \multirow{2}{*}{0.99861377} &\multirow{2}{*}{18.5801} & \multirow{2}{*}{1.1658e-04}\\ 
    && $(15.32, 10^{-5})$-DP &&&&\\ \cline{2-7}
            &  &$10^{-4}$ & \multirow{2}{*}{9.9565e-04} & \multirow{2}{*}{0.99996322} & \multirow{2}{*}{30.0189} & \multirow{2}{*}{8.6647e-06}\\
    && $(484.48, 10^{-5})$-DP &&&&\\
            && $10^{-3}$  &\multirow{2}{*}{2.3910e-03} &\multirow{2}{*}{0.99988467}& \multirow{2}{*}{26.2142} &\multirow{2}{*}{1.8149e-05}\\
    &DSG-& $(153.21, 10^{-5})$-DP &&&&\\
            &based& $10^{-2}$  & \multirow{2}{*}{5.7095e-03}&\multirow{2}{*}{0.99954313}& \multirow{2}{*}{22.4340} &\multirow{2}{*}{8.1432e-05} \\
    && $(48.45, 10^{-5})$-DP &&&&\\
            &  & $10^{-1}$  &\multirow{2}{*}{6.9409e-02} & \multirow{2}{*}{0.99145848} & \multirow{2}{*}{11.5858} &\multirow{2}{*}{5.7867e-04} \\
    && $(15.32, 10^{-5})$-DP &&&&\\
    \bottomrule
    \end{tabular}
    \caption{Quantitative analysis on reconstruction quality with noisy gradients. The similarity metrics are between the reconstructed image $\hat{\mathbf{x}}_0(\mathbf{x}_t)$ at peak performance and the target image.}
    \label{defense1}
\end{table*}

\subsubsection{Vulnerability of Different Machine Learning Models.}
Existing research has highlighted that different neural network models exhibit varying levels of susceptibility to the same reconstruction attack. However, a comprehensive explanation for this phenomenon has been lacking. In our analysis of Theorem \ref{lowerboundDSG}, we introduce a metric called Reconstruction Vulnerability ($RV$), which quantifies a model's susceptibility to reconstruction attacks. This metric provides a theoretical foundation for understanding why certain models are more vulnerable to privacy leakage than others.

To empirically validate the effectiveness of the $RV$ metric, we conduct experiments on a diverse set of models: MLP\_1, MLP\_2, MLP\_3, CNN, and ResNet9. All these models are parameterized with random fixed weights. We estimate the $RV$ values for each of these models using our proposed estimation method. The models were then subjected to the same reconstruction attack algorithms described earlier. We computed various metrics to quantify the differences between the reconstructed images and the original private images. The results are presented in Figure \ref{RVfigure} and Table \ref{defense-rv}, where models are ranked based on their estimated $RV$ values. 

\begin{figure}[htp!]
    \centering
    \vspace{-2mm}
    \subfigure[DPS-based Method]{
		\includegraphics[width=0.48\textwidth]{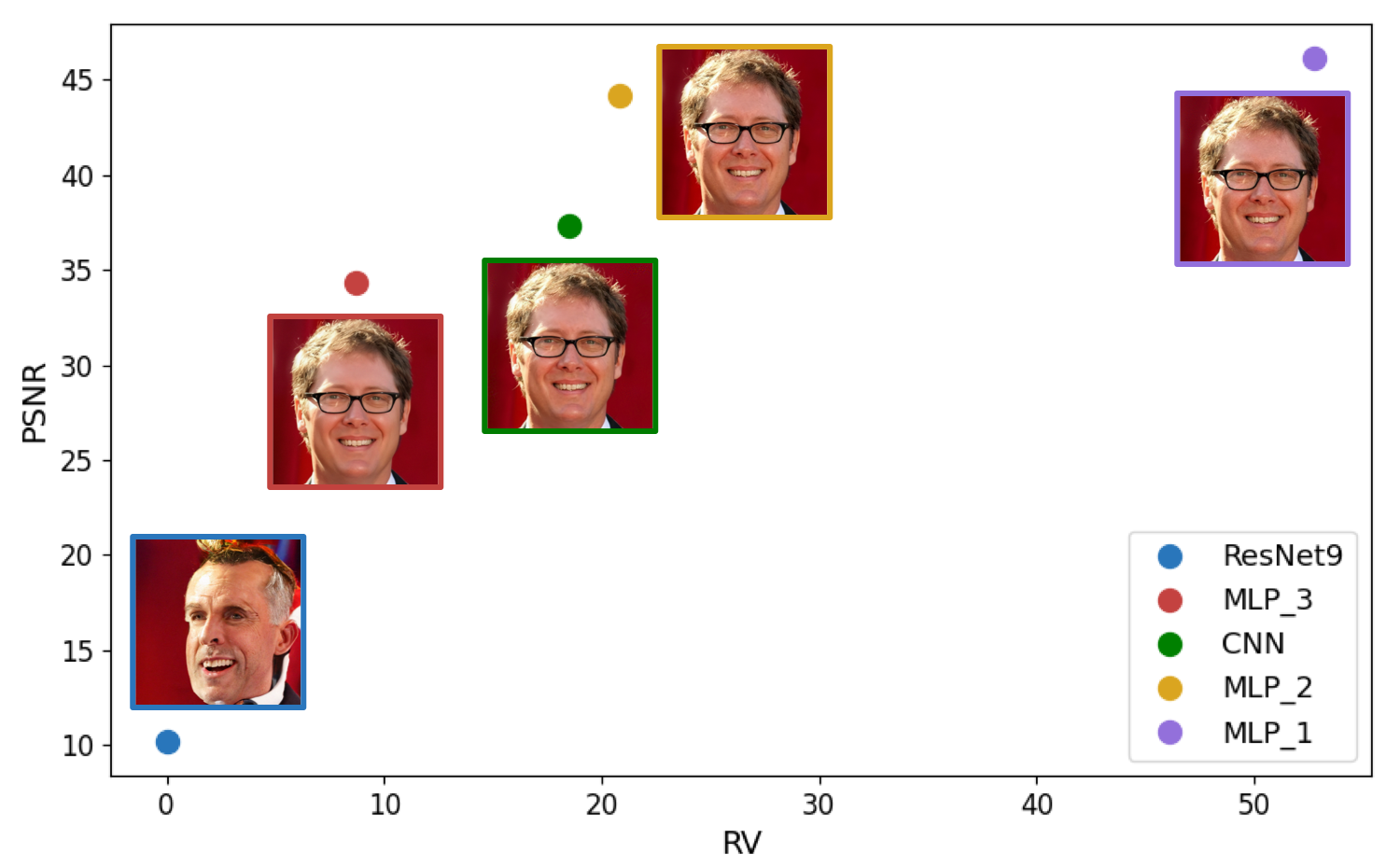}
		\label{DPS_model_RV}
    }
    \subfigure[DSG-based Method]{
		\includegraphics[width=0.48\textwidth]{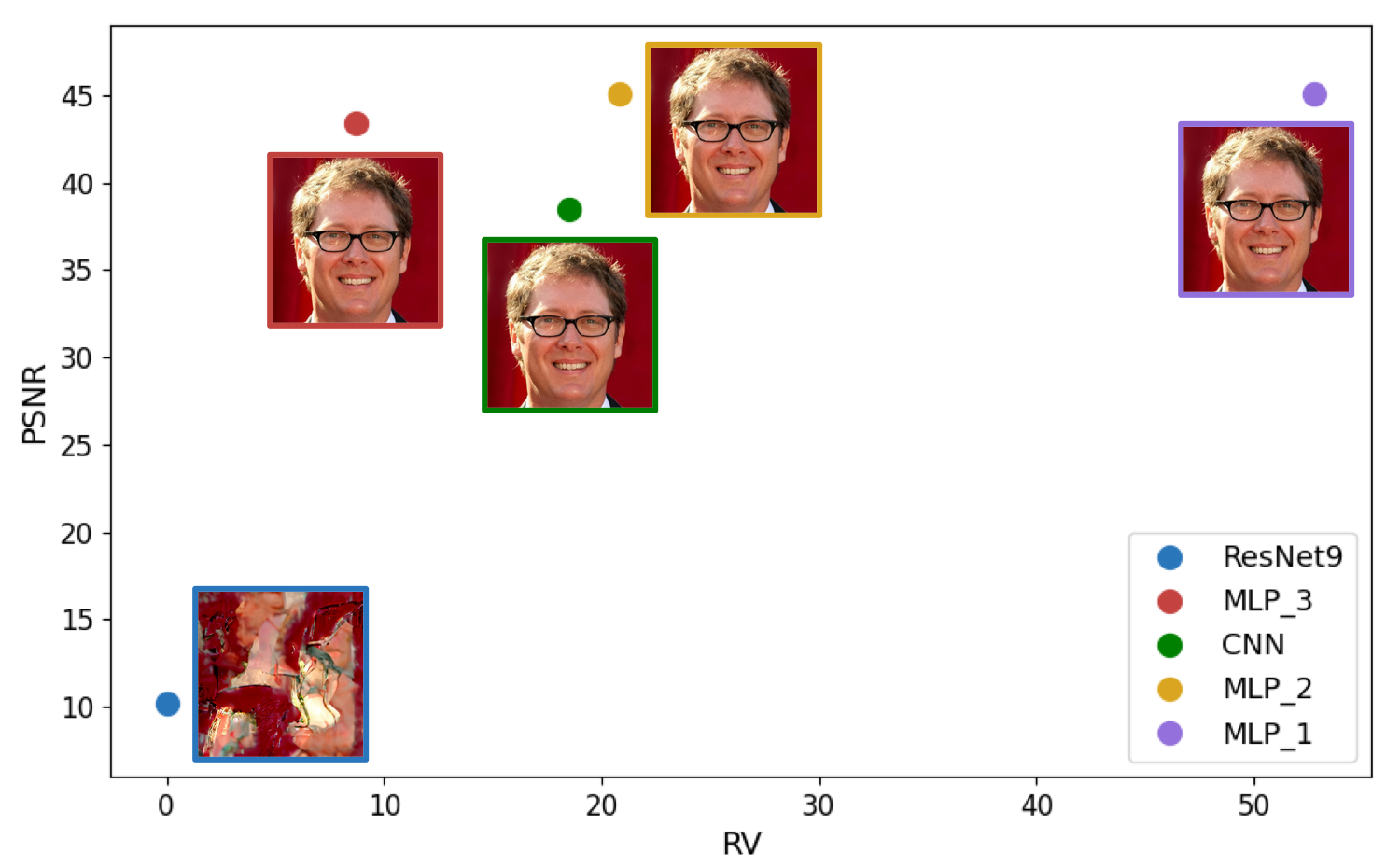}
		\label{DSG_model_RV}
    }
    \caption{Reconstruction results under different attacked models with different values of $RV$.}
    \label{RVfigure}
\end{figure}

\begin{table}[!htb]
    \centering
    \begin{tabular}{ccccccc}
    \hline
    \toprule
      \textbf{Method} & \textbf{Model}& \textbf{$RV$} & \textbf{MSE$\downarrow$} & \textbf{SSIM$\uparrow$} & \textbf{PSNR$\uparrow$} & \textbf{LPIPS$\downarrow$} \\
    \midrule
      \multirow{5}{*}{DPS-based} & ResNet9 & 0.0543 & 9.615e-02  & 0.98741 & 10.170 & 3.8624e-04\\
            & MLP\_3 & 8.7102 & 3.709e-04 & 0.99997 & 34.3071 & 2.9166e-06 \\
            & CNN & 18.5093  & 1.840e-04 & 0.99999 & 37.3518 & 1.1132e-06\\
            & MLP\_2 & 20.8503 & 3.822e-05 & 0.99999 & 44.1768 & 8.1932e-08\\ 
            & MLP\_1 & 52.7335 & 2.426e-05 & 0.99999 & 46.1509 & 5.6767e-08\\\midrule
      \multirow{5}{*}{DSG-based} & ResNet9 & 0.0543 & 9.500e-02  & 0.98712 & 10.223 & 4.1293e-04\\
            & MLP\_3 & 8.7102 & 3.125e-05 & 0.99999 & 45.052 & 9.8375e-08 \\
            & CNN & 18.5093  & 1.404e-04 & 0.99999 & 38.528 & 4.2582e-07\\
            & MLP\_2 & 20.8503 & 3.075e-05 & 0.99999 & 45.122 & 8.9244e-08\\ 
            & MLP\_1 & 52.7335 & 4.547e-05 & 0.99999 & 43.422 & 1.5980e-07\\
    \bottomrule
    \end{tabular}
    \caption{Quantitative analysis on reconstruction quality with different attacked models, corresponding to different $RV$.}
    \label{defense-rv}
\end{table}

Our experimental results reveal a significant trend: \textbf{models with larger $RV$ values tend to produce higher-quality reconstructed images}, indicating a higher risk of privacy leakage. This observation suggests a positive correlation between model's $RV$ value and its vulnerability to reconstruction attacks. High $RV$ Models like MLP\_1 yield reconstructed images with high PSNR and SSIM values, closely resembling the originals. Low MSE values further indicate minimal differences between the reconstructed and original images. Low $RV$ Models like ResNet9 produce lower-quality reconstructions, with lower PSNR and SSIM values and higher MSE, indicating greater dissimilarity from the original image.

While the correlation is not perfectly linear—due to factors such as model architecture complexity and training dynamics—the overall trend supports the relationship between $RV$ values and reconstruction quality. Models with higher $RV$ values are significantly more susceptible to reconstruction attacks.

\textbf{Understanding $RV$ as a Vulnerability Indicator.} The $RV$ metric encapsulates critical factors contributing to a model's vulnerability. From the aspect of gradient sensitivity, we think that models with higher $RV$ values may have gradients that are more informative about the input data, making them more exploitable by reconstruction algorithms. From the aspect of model complexity, we think that simple models (e.g., MLP\_1) may inadvertently "remember" more detailed information about the inputs in gradients.

\textbf{Practical Utility of $RV$ Metric.} The $RV$ metric serves as a practical and effective indicator for evaluating the vulnerability of different models to reconstruction attacks:

\textit{Model Selection}: Practitioners can use $RV$ values to select models with lower vulnerability for applications where data privacy is paramount.

\textit{Risk Assessment}: Organizations can assess the privacy risks associated with deploying specific models in sensitive environments.

\textit{Defense Mechanism Development}: Understanding $RV$ can guide the design of targeted defense strategies, such as gradient perturbation or architecture modifications, to mitigate privacy risks.

However, we think there are still some limitations concerning $RV$. The first is non-Perfect correlation. The absence of a perfect positive correlation indicates that other factors may also influence vulnerability, such as training data distribution and regularization techniques. The second is how to estimate $RV$ more accurately in practice. While we proposed a method to estimate $RV$ values, refining this estimation for different types of models and attack scenarios remains an area for future research.

In summary, our findings contribute a significant advancement in understanding model vulnerability to reconstruction attacks by introducing the Reconstruction Vulnerability ($RV$) metric. By empirically demonstrating that models with higher $RV$ values are more susceptible to producing high-quality reconstructed images. These results bridge a critical gap in the field and opens avenues for developing more robust models and defense mechanisms against reconstruction attacks. Future work will focus on refining the $RV$ estimation methods and exploring additional factors that influence model vulnerability.

\subsubsection{Impacts of Step Size and Guidance Rate}
In this section, we examine how the step size and the guidance rate in our proposed reconstruction attack methods—specifically, the step size $\zeta_i$ in the DPS-based method and the guidance rate $m_r$ in the DSG-based method—affect the quality of the reconstructed images. By systematically varying these parameters and evaluating the resulting metrics of the reconstruction quality, we aim to provide insights into optimizing these methods for effective image reconstruction from gradients.

\textbf{Impact of Step Size $\zeta_i$ in DPS-Based Method.} Table \ref{scale_DPS} shows the effect of varying the step size $\zeta_i$ on the reconstruction quality in the DPS-based method. The step size $\zeta_i$ in the DPS-based method controls the scale of the condition guidance for the intermediate results of the diffusion model's reverse process that minimizes the difference between the target gradient $g$ and the gradient $g_0(\hat{\mathbf{x}}_0(\mathbf{x}_t))$ computed from the current prediction $\hat{\mathbf{x}}_0(\mathbf{x}_t)$. Increasing $\zeta_i$ amplifies the influence of the gradient information in guiding the reconstruction. The finding is that the best reconstruction metrics are observed around the scale of 230, where the MSE is lowest ($1.0550e-04$), SSIM is highest ($0.99999696$), PSNR is highest ($39.7673$), and LPIPS is lowest ($1.9310e-07$). At this point, the step size is sufficient to effectively minimize the gradient difference without causing instability in the updates. Beyond the optimal scale, further increasing $\zeta_i$ leads to a deterioration in reconstruction quality. This is likely due to overstepping in the gradient descent updates, causing the algorithm to overshoot the minimum of the loss function. Large step sizes can introduce oscillations or divergence in the optimization process, resulting in poorer reconstructions.

\begin{table}[!htb]
    
    \centering
    \begin{tabular}{ccccc}
    \hline
    \toprule
        \textbf{Step Size} & \textbf{MSE$\downarrow$} & \textbf{SSIM$\uparrow$} & \textbf{PSNR$\uparrow$} & \textbf{LPIPS$\downarrow$} \\
    \midrule
             30 & 1.5701e-03  & 0.99994308 & 28.0407 & 8.4652e-06\\
             70 & 7.5194e-04 & 0.99997950 & 31.2382 & 2.2519e-06 \\
             110  & 2.6727e-04 & 0.99999362 &35.7306 & 4.8490e-07\\
             150 & 2.4821e-04 & 0.99999398 & 36.0518 & 5.2666e-07\\ 
             190 & 1.5373e-04 & 0.99999577 & 38.1324 & 2.7950e-07\\
             230 & \textbf{1.0550e-04} & \textbf{0.99999696} & \textbf{39.7673} & \textbf{1.9310e-07}\\
             270 & 1.5543e-04 & 0.99999601 & 38.0847 & 2.4074e-07\\
             310 & 6.0308e-03 & 0.99969864 & 22.1962 & 2.9200e-05\\
    \bottomrule
    \end{tabular}
    \caption{The impact of step size $\zeta_i$ in DPS-based method on the reconstruction quality.}
    \label{scale_DPS}
\end{table}

\textbf{Impact of Guidance Rate $m_r$ in DSG-Based Method.} Table \ref{scale_DSG} presents the effect of varying the guidance rate $m_r$ on the reconstruction quality in the DSG-based method. The guidance rate values range from 0.01 to 1.00. The best reconstruction metrics are observed at $m_r = 0.20$, where MSE is lowest ($7.7215e-05$), SSIM is highest ($0.99999720$), PSNR is highest ($41.1230$), and LPIPS is lowest ($2.0519e-07$). 

The guidance rate $m_r$ in the DSG-based method determines the extent to which the conditional guidance $d^{*}$ influences the update direction $d_m$ during the denoising process. At low values of $m_r$, the influence of the conditional guidance is minimal, and the updates are dominated by the stochastic component $d^{\text{sample}}$. This can result in less effective guidance towards the original image. Increasing $m_r$ enhances the contribution of the conditional guidance, improving the alignment with the gradient information and leading to better reconstruction quality. An optimal value of $m_r = 0.20$ provides a balance where the guidance effectively steers the reconstruction without overwhelming the stochastic exploration inherent in the diffusion process. Beyond the optimal point, further increasing $m_r$ diminishes the stochastic component's role, which is essential for exploring the solution space and avoiding local minima. Excessive reliance on the guidance can cause the updates to become overly deterministic and potentially lead to convergence issues or suboptimal reconstructions. High guidance rates may also amplify any noise or inaccuracies in the gradient estimates, causing the reconstruction to overfit to erroneous gradient information and degrade in reconstruction quality.

\begin{table}[!htb]
    \centering
    \begin{tabular}{ccccc}
    \hline
    \toprule
        \textbf{Guidance} & \multirow{2}{*}{\textbf{MSE$\downarrow$}} & \multirow{2}{*}{\textbf{SSIM$\uparrow$}} & \multirow{2}{*}{\textbf{PSNR$\uparrow$}} & \multirow{2}{*}{\textbf{LPIPS$\downarrow$}} \\
        \textbf{Scale} &&&&\\
    \midrule
             0.01 & 2.0406e-03 & 0.99991947 & 26.9023 & 6.9703e-06\\
             0.05 & 3.8725e-04 & 0.99998808 & 34.1201 & 1.1588e-06 \\
             0.10 & 2.1778e-04 & 0.99999398 & 36.6198 & 4.1730e-07\\
             0.15 & 1.3556e-04 & 0.99999613 & 38.6787 & 2.4473e-07\\ 
             0.20 & \textbf{7.7215e-05} & \textbf{0.99999720} & \textbf{41.1230} & \textbf{2.0519e-07}\\
             0.30 & 1.4769e-04 & 0.99999571 & 38.3065 & 2.8713e-07\\
             0.40 & 1.0778e-04 & 0.99999684 & 39.6746 & 1.9574e-07\\
             0.60 & 2.0407e-04 & 0.99999332 & 36.9023 & 9.0661e-07\\
             0.80 & 2.7478e-04 & 0.99998826 & 35.6101 & 1.1251e-06\\
             1.00 & 3.1045e-04 & 0.99998742 & 35.0800 & 1.9515e-06\\
    \bottomrule
    \end{tabular}
    \caption{The impact of guidance scale $m_r$ in DSG-based method on the reconstruction quality.}
    \label{scale_DSG}
\end{table}

In summary, both the step size $\zeta_i$ in the DPS-based method and the guidance rate $m_r$ in the DSG-based method have optimal values that maximize reconstruction quality. These optimal values represent a trade-off between effectively utilizing gradient information and maintaining stability in the reconstruction process. Using values that are too low may result in insufficient guidance, while values that are too high can introduce instability or overfitting. This may be correlated with the role of stochasticity in diffusion models. The stochastic components in the diffusion process play a vital role in ensuring diversity and robustness in the reconstructions. Maintaining a balance between deterministic guidance and stochastic exploration is essential for optimal reconstruction performance.

\subsubsection{Comparison With Other Baselines.}

We first compare our proposed methods with other baselines without adding differentially private noise to gradients. The quantitative results are presented in Table \ref{reconstruction_similarity}. The resolution of the reconstructed images of DLG\cite{b5}, FinetunedDiff\cite{b25}, our proposed DPS-based method and DSG-based method is $256 \times 256$ pixels while other baselines can only reconstruct images of $64 \times 64$ pixels. As shown in Table \ref{reconstruction_similarity}, our proposed two methods can guide the diffusion model to extract high-resolution images and the reconstruction performances are even better. This presents severe privacy leakage risks if the adversary steals the original gradients.

\begin{table*}[!htp]
    \small
    \scalebox{0.76}{
        \begin{tabular}{cccccccccc}
            \hline
            \toprule
            \textbf{\multirow{2}{*}{Metrics}} & \textbf{\multirow{2}{*}{DLG}} & \textbf{\multirow{2}{*}{IG}} & \textbf{\multirow{2}{*}{GI}} & \textbf{\multirow{2}{*}{GIAS}} & \textbf{\multirow{2}{*}{GGL}} & \textbf{\multirow{2}{*}{GIFD}} & \textbf{Finetuned} & \textbf{DPS-} & \textbf{DSG-}  \\ 
            &&&&&&&\textbf{Diff}& \textbf{based} &\textbf{based} \\
            \midrule
            \textbf{MSE$\downarrow$} & 0.0480 & 0.0196 & 0.0223 & 0.0458 & 0.0179 & 0.0098 & 0.0030 & \textbf{0.0001} & \textbf{0.0001} \\ 
            \textbf{SSIM$\uparrow$} & 0.9979 & 0.9982 & 0.9978 & 0.9953 & 0.9987 & 0.9991 & 0.9999 & \textbf{0.9999}  & \textbf{0.9999} \\ 
            \textbf{PSNR$\uparrow$} & 13.1835 & 17.0756 & 16.5109 & 13.3885 & 17.4923 & 20.0534 & 25.2863 &  39.7673 & \textbf{41.1229} \\
            \textbf{LPIPS$\downarrow$} & 2.3453e-04 & 1.1612e-04 & 1.4274e-04 & 3.9351e-04 & 1.1937e-04 & 6.6672e-05 & 1.3781e-05 &  \textbf{1.9310e-07} & 2.0519e-07 \\ \bottomrule
        \end{tabular}
    }
   \caption{Quantitative evaluations of SOTA baselines and our methods.}
   \label{reconstruction_similarity}
\end{table*}

We also compare our proposed methods with other baselines under the framework of differential privacy by adding Gaussian noise to the gradients. Image reconstruction is performed using various baselines, with the noise variance fixed at 0.01. The corresponding quantitative results are presented in Table \ref{noise_baseline}. As shown, our DPS-based method achieves the highest PSNR, outperforming the optimal baseline, GGL, by 3.5087. Additionally, the PSNR of our DSG-based method exceeds that of GGL by 0.0727, demonstrating its superior performance. Despite the injection of Gaussian noise into the gradients, experimental results indicate that a potential risk of privacy leakage persists. This underscores that merely adding noise satisfying differential privacy is insufficient to fully safeguard the sensitive information embedded in the gradients. Therefore, it is essential to explore more robust methods for protecting gradient privacy in practical applications.

\begin{table*}[!h]
   \centering
   \scalebox{0.76}{
        \begin{tabular}{cccccccccc}
            \hline
            \toprule
            \textbf{Methods} & \textbf{DLG} & \textbf{IG} & \textbf{GI} & \textbf{GIAS} & \textbf{GGL} & \textbf{GIFD} & \textbf{FinetunedDiff} & \textbf{DPS-based} & \textbf{DSG-based}  \\ \midrule
            \textbf{PSNR$\uparrow$} & 6.5871 & 11.2677 & 10.4968 & 12.1276 & 14.8982 & 13.7118 & 11.6675 &  \textbf{18.4069} & 14.9709 \\
             \bottomrule
        \end{tabular}
    }
   \caption{Quantitative evaluations of SOTA baselines and our methods with Gaussian noise of variance 0.01 added to the original gradients.}
   \label{noise_baseline}
\end{table*}

\section{Mitigation Strategies}
Building on the experimental results and analyses presented, we propose several targeted mitigation strategies to defend against reconstruction attacks that exploit diffusion models. 

To reduce the risk of privacy leakage, we recommend designing or selecting model architectures that inherently exhibit lower Reconstruction Vulnerability ($RV$) values. Additionally, we suggest integrating $RV$ estimation into the training process to continuously monitor and minimize $RV$ in real time. This can be achieved by incorporating regularization techniques or modifying training objectives to penalize high $RV$ values, encouraging the model to learn representations that are less prone to leakage through gradients.

For gradient protection, introducing carefully designed perturbations into the gradients can disrupt attack algorithms that rely on precise gradient values and directions. Potential strategies include adding structured noise or perturbations to the gradients to mislead the reconstruction process and applying transformations that preserve the gradients' utility for model training while obscuring sensitive information required for reconstruction.

\section{Conclusion}
In this paper, we present two novel methods for leveraging conditional diffusion models to reconstruct private images from leaked or stolen gradients. These methods require minimal adjustments to the diffusion model’s process and do not rely on prior knowledge, making them potent tools for attackers. Our theoretical analysis and experimental results demonstrate the significant impact of differential privacy noise and model type on the quality of reconstructed images, revealing key insights into the relationship between the diffusion model’s denoising capabilities and differential privacy’s noise defense. Our work highlights the importance of understanding the adversarial interaction between differential privacy and diffusion models, providing a foundation for future research in enhancing privacy protection mechanisms in machine learning.

\bibliographystyle{named}
\bibliography{neurips_2024}

\end{document}